\documentclass{article}
\usepackage{arxiv}

\usepackage{booktabs}
\usepackage{amsthm}
\usepackage{amsfonts}
\usepackage[multiple]{footmisc}
\usepackage{booktabs}
\usepackage{dsfont}
\usepackage{siunitx}
\usepackage{multirow}
\usepackage{graphicx}
\usepackage{listings}
\usepackage{commath}
\usepackage{epstopdf}
\usepackage{hyperref}
\usepackage{color}
\usepackage{url}
\usepackage[font=bf]{caption}
\usepackage{alltt}
\usepackage{mathrsfs}
\usepackage{float}
\usepackage{subcaption}
\usepackage{graphicx,fancyvrb}
\usepackage[linesnumbered,ruled,vlined]{algorithm2e}
\usepackage{algpseudocode}
\usepackage{amsmath,amssymb}
\usepackage{booktabs}

\usepackage{mathtools, amssymb}
\usepackage{caption}
\usepackage{float}
\usepackage{capt-of}
\usepackage{booktabs}
\usepackage{colortbl}
\usepackage{xcolor}
\usepackage{xfrac}
\usepackage{footnote}

 \usepackage{enumitem}
\usepackage{array}
\usepackage{arydshln}
\usepackage{amsmath,amsfonts,amssymb}

\newcommand{\reva}[1]{#1}
\newcommand{\revb}[1]{#1}
\newcommand{\revc}[1]{#1}

\definecolor{mygreen}{rgb}{0,0.6,0}
\definecolor{mygray}{rgb}{0.5,0.5,0.5}
\definecolor{mymauve}{rgb}{0.58,0,0.82}

\usepackage{listings}
\lstnewenvironment{VerbatimText}[1][]{
    
    \lstset{fancyvrb=true,basicstyle=\footnotesize,captionpos=b,xleftmargin=2em,#1}
}{}

\usepackage{booktabs}
\usepackage{pgfplots}
\usepackage{pgfplotstable}

\PassOptionsToPackage{hyphens}{url}\usepackage{hyperref}

\newcommand{\cg}{ G}

\usepackage{amsmath}

\newcommand{\nec}{{\textsc{Nec}}}
\newcommand{\suf}{\textsc{Suf}}
\newcommand{\nsuf}{\textsc{NeSuf}}
\usepackage{bbm}

\newcommand{\cm}{{{ M }}}
\newcommand{\pcm}{\langle \cm, \pr(\mb u) \rangle}

\newcommand{\Do}{{{ \texttt{do} }}}
\newcommand{\Pa}{\mb{Pa}}

\DeclareMathOperator*{\argmin}{argmin}

\newcommand{\pr}{{\tt \mathrm{Pr}}}
\newcommand{\sys}{{\textsc{Lewis}}}

\newcommand{\ignore}[1]{}

\newcommand*{\rom}[1]{\expandafter\@slowromancap\romannumeral #1@}

\newcommand{\romila}[1]{{\texttt{\color{blue} Romila: [{#1}]}}}

\newcommand{\indep}{\mbox{$\perp\!\!\!\perp$}}

\newcommand{\RNum}[1]{\uppercase\expandafter{\romannumeral #1\relax}}

\newcommand{\mb}[1]{{\mathbf{#1}}}
\newtheorem{definition}{Definition}[section]
\newtheorem{remark}[definition]{Remark}

\newtheorem{example}[definition]{Example}
\newtheorem{col}[definition]{Colollary}

\newtheorem{prop}[definition]{Proposition}

\newcommand{\proj}[1]{{\Pi}}
\newcommand{\sel}[1]{{\sigma}}

\newcommand{\cut}[1]{}
\newcommand{\eat}[1]{}
\usepackage[multiple]{footmisc}

\newcommand{\defeq}{\stackrel{\text{def}}{=}}

\newcommand{\positiveuser}{Irrfan}
\newcommand{\negativeuser}{Maeve}

\title{Explaining Black-Box Algorithms Using Probabilistic Contrastive Counterfactuals}
\author{
  Sainyam Galhotra\textsuperscript{*}\\
   University of Massachusetts Amherst\\
   \texttt{sainyam@cs.umass.edu} \\
   \And
 Romila Pradhan   \thanks{These authors contributed equally to this work.}\\
  University of California, San Diego\\
   \texttt{rpradhan@ucsd.edu} \\
  \And
  Babak Salimi \\
  University of California, San Diego\\
  \texttt{bsalimi@ucsd.edu} 
}

\date{}

\begin{document}
\maketitle
\pagestyle{headings} 

\begin{abstract}


There has been a recent resurgence of interest in {\em explainable artificial intelligence} (XAI) that aims to reduce the opaqueness of AI-based decision-making systems, allowing humans to scrunitize and trust them.
%
Prior work in this context has focused on the attribution of {\em responsibility} for an algorithm’s decisions to its inputs wherein responsibility is typically approached as a purely {\em associational} concept. In this paper, we propose a principled causality-based approach for explaining black-box decision-making systems that addresses limitations of existing methods in XAI. At the core of our framework lies {\em probabilistic contrastive counterfactuals}, a concept that can be traced back to philosophical, cognitive, and social foundations of theories on how humans generate and select explanations. We show how such counterfactuals can quantify the {\em direct} and {\em indirect} influences of a variable on decisions made by an algorithm, and provide {\em actionable recourse} for individuals negatively affected by the algorithm's decision. Unlike prior work, our system, \sys{}: (1)~can compute
provably effective explanations and recourse at local, global and contextual levels; (2)~is designed to work with users with varying levels of background knowledge of the underlying causal model; and (3)~makes no assumptions about the internals of an algorithmic system except for the availability of its input-output data. We empirically evaluate \sys\ on three real-world datasets and show that it generates human-understandable explanations that improve upon state-of-the-art approaches in XAI, including the popular LIME and SHAP. Experiments on synthetic data further demonstrate the correctness of \sys's explanations and the scalability of its recourse algorithm.
\end{abstract}

\maketitle

\vspace{-.1cm}
\section{Introduction}

Algorithmic decision-making systems are increasingly used to automate consequential decisions, such as lending, assessing job applications, informing release on parole, and prescribing life-altering medications. There is growing concern that the opacity of these systems can inflict harm to stakeholders distributed across different segments of society. These calls for transparency created a resurgence of interest in {\em explainable artificial intelligence} (XAI), which aims to provide human-understandable explanations of outcomes or processes of algorithmic decision-making systems (see \cite{guidotti2018survey,mittelstadt2019explaining,molnar2020interpretable} for recent surveys).  

{\em Effective explanations} should serve the following
purposes: (1) help to build trust by providing a mechanism for {\em normative evaluation} of an algorithmic system, ensuring different stakeholders that the system's decision rules are justifiable~\cite{selbst2018intuitive}; and (2) provide users with an {\em actionable  recourse}
to change the results of algorithms in the future~\cite{DBLP:books/sp/Berk19,wachter2017counterfactual,venkatasubramanian2020philosophical}. Existing methods in XAI can be broadly categorized based on whether explainability is achieved by design ({\em intrinsic}) or by post factum system analysis ({\em post hoc}), and whether the methods assume access to system internals ({\em model dependent}) or can be applied to any black-box algorithmic system ({\em model agnostic}). 
\begin{figure*}
    \centering
    \includegraphics[scale=0.49]{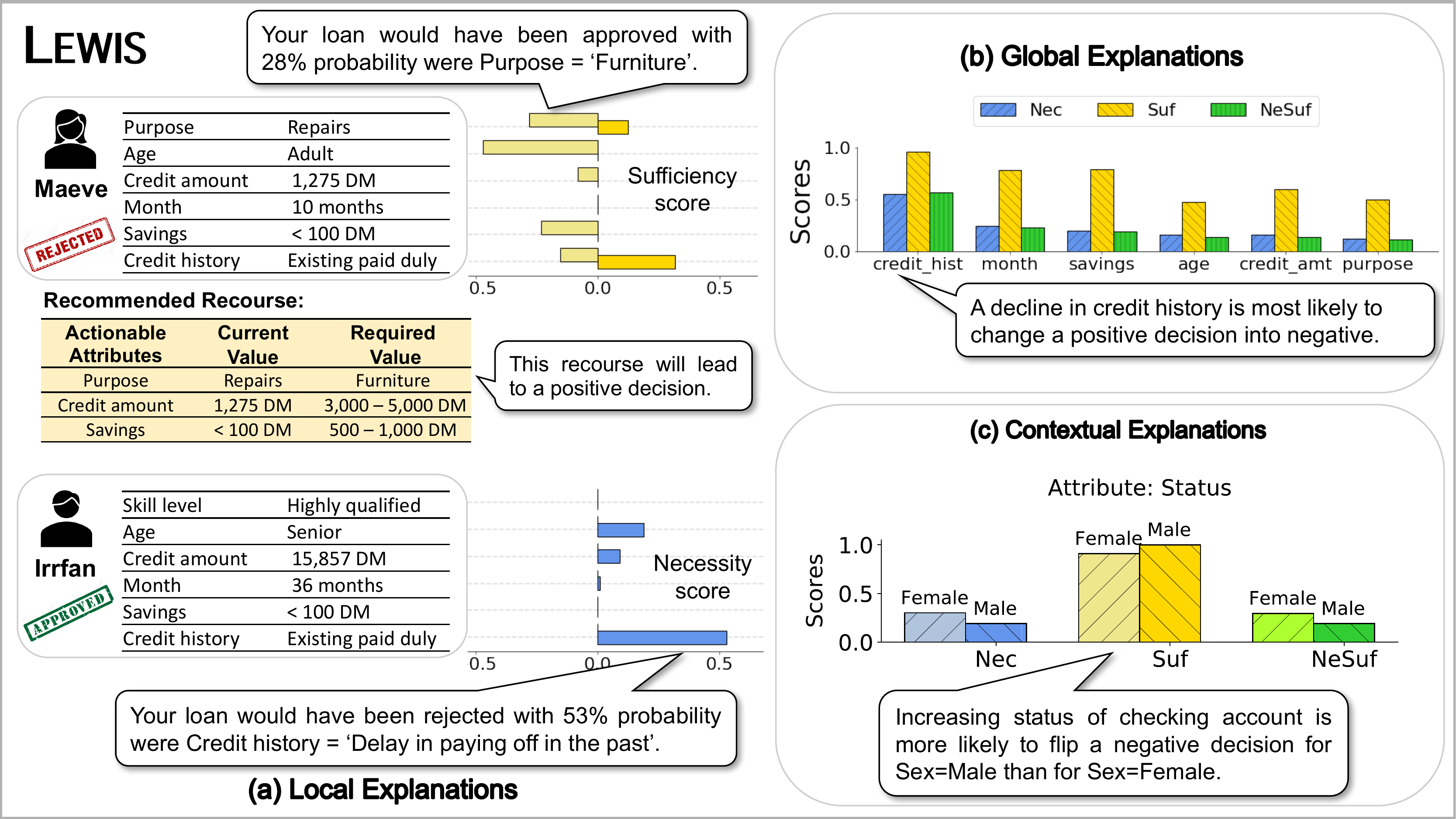}
    \caption{{An overview of explanations generated by $\sys$ for a loan approval algorithm built using the UCI German credit dataset (see Section~\ref{sec:exp} for details). Given a black-box classification algorithm, \sys{} generates: (a)~Local explanations, that explain the algorithm's output for an individual; (b)~Global explanations, that explain the algorithm's behavior across different attributes; and (c)~Contextual explanations, that explain the algorithms's predictions for a sub-population of individiuals.}}
    \label{fig:ravan}
\end{figure*}
%




In this work, we address post hoc and model-agnostic explanation methods that are applicable to any proprietary black-box algorithm.  Prior work in this context has focused on the attribution of {\em responsibility} of an algorithm's decisions to its inputs. These approaches include methods for quantifying the {\em global} (population-level) or {\em local} (individual-level) {\em influence} of an algorithm's input on its output
\cite{friedman2001greedy,goldstein2015peeking, apley2016visualizing,hooker2004discovering,greenwell2018simple,fisher2018model, lundberg2017unified,lundberg2018consistent,datta2016algorithmic}; they also include methods based on {\em surrogate explainability}, which search for a simple and interpretable model (such as a decision tree or a linear model) that mimics the behaviour of a black-box algorithm~\cite{ribeiro2016should,ribeiro2018anchors}.  However, these methods can produce incorrect and misleading explanations primarily because they focus on the {\em correlation} between the input and output of algorithms as opposed to their {\em causal} relationship~\cite{hooker2019please,kumar2020problems,frye2019asymmetric,guidotti2018survey,molnar2020interpretable,alvarez2018robustness}. Furthermore,  several recent works have
argued for the use of {\em counterfactual explanations}, which are typically obtained by considering the smallest perturbation in an algorithm's input that can lead to the algorithm's desired outcome~\cite{wachter2017counterfactual,laugel2017inverse,ustun2019actionable,mahajan2019preserving,mothilal2020explaining}. However, due to the causal dependency between variables, these perturbations are not translatable into real-world interventions and therefore fail to generate insights that are actionable in the real world~\cite{barocas2020hidden,karimi2019model,sokol2019counterfactual,mahajan2019preserving,karimi2020algorithmic}.

This paper describes a \textbf{new causality-based framework} for generating post-hoc explanations for black-box decision-making algorithms \textbf{that unifies existing methods in XAI and addresses their limitations.}  Our system, \sys, \footnote{Our system is named after David Lewis (1941–2001), who made significant contributions to modern theories of causality and explanations in terms of counterfactuals. In his essay on causal explanation \cite{lewis1986causal}, Lewis argued that ``to explain an event is to provide some information about its causal history." He further highlighted the role of counterfactual contrasts in explanations when he wrote, ``One way to indicate what sort of explanatory information is wanted is through the use of contrastive why-questions $\ldots$ [where] information is requested about the difference between the
actualized causal history of the explanandum and the unactualized
causal histories of its unactualized alternatives [(termed as ``foils" by Peter Lipton~\cite{lipton1990contrastive})]. Why did I visit
Melbourne in 1979, rather than Oxford or Uppsala or Wellington?" } reconciles the aforementioned objectives of XAI by: (1) providing insights into what {\em causes} an algorithm's decisions at the global, local and contextual (sub-population) levels, and (2) generating  actionable recourse translatable into real-world interventions. At the heart of our proposal are {\em probabilistic contrastive counterfactuals} of the following form:
%
\begin{equation}
\text{\parbox{.72\textwidth}{``For individual(s) with attribute(s) \textsf{$<$actual-value$>$} for whom an  algorithm made the decision \textsf{$<$actual-outcome$>$}, the decision would have been \textsf{$<$foil-outcome$>$} with {\em probability} \textsf{$<$score$>$} had the attribute been \textsf{$<$counterfactual-value$>$}."}} \label{eq:ppc}
\end{equation}
Contrastive counterfactuals are at the core of the philosophical, cognitive, and social foundations of theories that address how humans generate and select explanations~\cite{de2017people,gerstenberg2015whether,pearl2009causality,lipton1990contrastive,woodward2005making,grynaviski2013contrasts,morton2013contrastive}. Their probabilistic interpretation has been formalized and studied extensively in AI, biostatistics, political science, epistemology,  biology and legal reasoning~\cite{greenland1999relation, robins1989probability, greenland1999epidemiology, tian2000probabilities, greenland1999relation, robertson1996common, cox1984probability, pearl2009causality,grynaviski2013contrasts,mandel2005counterfactual}.  While their importance in achieving the objectives of XAI has been recognized in the literature ~\cite{miller2019explanation}, very few attempts have been made to operationalize causality-based contrastive counterfactuals for XAI. The following example illustrates how \sys\ employs contrastive counterfactuals to generate different types of explanations. 
\begin{example}
\label{ex:exp} Consider the black-box loan-approval algorithm in Figure~\ref{fig:ravan} for which \sys{} generates different kinds of explanations. For local explanations, \sys\ ranks attributes in terms of their {\em causal} responsibility to the algorithm's decision. For individuals whose loans were rejected, the responsibility of an attribute is measured by its {\em sufficiency} score, defined as ``the probability that the algorithm's decision would have been positive if that attribute had a counterfactual value". {For \negativeuser, the sufficiency score of $28\%$ for purpose of loan means that if purpose were `Furniture', \negativeuser's loan would have been approved with a  $28\%$ probability.} For individuals whose loans were approved, the responsibility of an attribute is measured by its {\em necessity} score, defined as ``the probability that the algorithm's decision would have been negative if that attribute had a counterfactual value." For \positiveuser, the necessity score of $53\%$ for credit history means that had credit history been worse, \positiveuser\ would have been denied the loan with a $53\%$ probability. 
Furthermore, individuals with a negative decision, such as \negativeuser, would want to know the actions they could take that would likely change the algorithm's decision. For such users, \sys\ suggests the minimal causal interventions on the set of actionable attributes that are sufficient, with high probability, to   change the algorithm's decision in the future. Additionally, \sys\ generates insights about the algorithm's {\em global} behavior with respect to each attribute by computing its necessity, sufficiency, and necessity and sufficiency scores at the population level. For instance, a higher necessity score for credit history indicates that a decline in its value is more likely to reverse a positive decision than a lower value of savings; a lower sufficiency score for age indicates that increasing it is less likely to overturn a negative decision compared to credit history or savings. By further customizing the scores for a {\em context} or sub-population of individuals that share some attributes, \sys\ illuminates the {\em contextual} behavior of the algorithm in different sub-populations. In Figure~\ref{fig:ravan}, \sys\ indicates that increasing the \texttt{status} is more likely to reverse a negative decision for \{\texttt{sex=Male}\} than for \{\texttt{sex=Female}\}.

\end{example}
{To compute these scores, \sys{} relies on the ordinal importance of  attribute values e.g., higher savings are more likely to be granted a loan than lower savings.  In case the attribute values do not possess a natural ordering or the ordering is not known apriori, \sys{} infers it from the output of the black-box algorithm (more in Section~\ref{sec:identif}).}
\vspace{.2cm}

\par
{\bf Our contributions.} This paper proposes a principled approach for explaining black-box decision-making systems using probabilistic contrastive counterfactuals.  Key contributions include:
\vspace{-0.1cm}
\begin{enumerate}[leftmargin=0.35cm,align=left,labelsep=-.35cm]



\item Adopting standard definitions of sufficient and necessary causation based on contrastive counterfactuals to propose \textbf{novel probabilistic measures, called \textit{necessity scores} and \textit{sufficiency scores}}, which respectively quantify the extent to which an attribute is necessary and sufficient for an algorithm's decision ({Section~\ref{sec:probcauses}}). We show that these measures play unique, complementary roles in generating effective explanations for algorithmic systems. While the necessity score addresses the {\em attribution} of causal responsibility of an algorithm's decisions to an attribute, sufficient score addresses the tendency of an attribute to {\em produce} the desired algorithmic outcome. 
 
\item Demonstrating that our newly proposed measures can generate a \textbf{wide range of explanations} for algorithmic systems that quantify the necessity and sufficiency of attributes that \reva{{\bf implicitly} or {\bf explicitly} influence an algorithm's decision making process}~(Section~\ref{sec:sysexp}). More importantly, \sys\ generates \textbf{contextual explanations} at global or local levels and for a user-defined sub-population.  

\item Showing that the problem of generating {\bf actionable recourse} can be framed as an optimization problem that searches for a  {\bf  minimal intervention} on a pre-specified set of actionable variables that have a high {\bf  sufficiency score} for producing the algorithm's desired future outcome.




\item Establishing conditions under which\textbf{ the class of probabilistic contrastive counterfactuals we use can be bounded and estimated using historical data} ({Section~\ref{sec:identif}}). 
Unlike previous attempts to generate actionable recourse using counterfactual reasoning, \sys\ leverages established bounds and integer programming to generate reliable recourse under partial background knowledge on the underlying causal models~({ Section~\ref{sec:computingrecourse}}). 



\item Comparing $\sys$ to state-of-the-art methods in XAI ({Sections~\ref{sec:exp} and \ref{sec:related}}). We present an \textbf{end-to-end experimental evaluation on both real and synthetic data}. In real datasets, we show that \sys\ generates intuitive and actionable explanations that are consistent with insights from existing literature and surpass state-of-the-art methods in XAI. Evaluation on synthetic data demonstrates the accuracy and correctness of the explanation scores and actionable recourse  that \sys\ generates.   
\end{enumerate}

\vspace{-0.3cm}
\section{Preliminaries}
\label{sec:back}
The notation we use in this paper is summarized in Table~\ref{tab:notations}. We denote variables by uppercase letters,
$X, Y, Z, $ $V$; their values with lowercase letters, $x,y,z, v$; and sets of variables or values using boldface ($\mb X$ or
$\mb x$).  The domain of a variable $X$ is $Dom(X)$, and the domain of
a set of variables is $Dom(\mb X) = \prod_{X\in \mb X} Dom(X)$. All domains are discrete and finite; continuous domains are assumed to be binned. We use $\pr(\mb x)$ to represent a joint probability distribution $\pr(\mb X=\mb x)$. The basic semantic framework of our proposal rests on probabilistic causal models \cite{pearl2009causality}, which we review next.

\begin{table}
\footnotesize
	\centering
	\begin{tabular}{|l|l|} \hline
		\textbf{    Symbol} & \textbf{Meaning} \\ \hline
		$X, Y, Z$ & attributes (variables)\\
		$\mb{X}, \mb{Y}, \mb{Z}$ & sets of attributes \\
		$Dom(X), Dom(\mb{X})$ & their domains \\
		$x \in Dom(X)$ & an attribute value \\ 
		$\mb{x} \in Dom(\mb{X})$& a tuple of attribute values \\
		$\mb{k} \in Dom(\mb{K})$& a tuple of context attribute values \\
		$\cg$ & causal diagram \\
		 $\pcm$ & probabilistic causal model\\
		 		 $O_{\mb X \leftarrow\mb x}$ & potential outcome\\
		 $\pr(\mb V= \mb v), \pr(\mb v)$ &  joint probability distribution\\
		 $\pr(o_{\mb X \leftarrow\mb x})$ & abbreviates $\pr(O_{\mb X \leftarrow\mb x}=o)$\\
	\hline
	\end{tabular}
	\caption{      {\bf Notation used in this paper.}}
	\label{tab:notations}
\end{table}



\par {\bf Probabilistic causal models.}  A {\em probabilistic causal model} (PCM) is a tuple $\pcm$, where $\cm = \langle \mb U, \mb V, \mb F\rangle$ is a {\em causal model} consisting of a set of \emph{observable or endogenous} variables $\mb V$ and a set of
\emph{background or exogenous} variables $\mb U$ that are outside of the model, and $\mb F = (F_X)_{X \in \mb V}$
is a set of \emph{structural equations} of the form
$F_X : Dom(\Pa_{\mb V}(X)) \times Dom(\Pa_{\mb U}(X)) \rightarrow Dom(X)$, where
$\Pa_{\mb U}(X) \subseteq \mb{U}$ and $\Pa_{\mb V}(X) \subseteq \mb V-\set{X}$
are called \textit{exogenous parents} and \textit{endogenous parents} of $X$, 
respectively. The values of $\mb U$ are drawn from the distribution $\pr(\mb u)$. A PCM $\pcm$ can be represented as a directed graph
$\cg=\langle \mb V, \mb E \rangle$, called a {\em causal diagram}, where each node represents a variable, and there are directed edges from the
elements of $\Pa_{\mb U}(X) \cup \Pa_{\mb V}(X)$ to $X$. We say a variable $Z$ is a {\em descendant} of another variable $X$ if $Z$ is {\em caused} (either {\em directly} or {\em indirectly}) by $X$, i.e., if there is a directed edge or path from $X$ to $Z$ in $G$; otherwise, we say that $Z$ is a {\em non-descendant} of $X$. 
\par 
{\bf Interventions and potential outcomes.}
An {\em intervention} or an {\em action} on a set of variables
$\mb X \subseteq \mb V$, denoted 
$\mb X \leftarrow \mb x$, is an operation that {\em modifies} the underlying causal model by replacing the structural equations associated with $\mb X$ with a constant $\mb x \in Dom(\mb X)$. The {\em potential outcome} of a variable $Y$ after the intervention $\mb X \leftarrow \mb x$ in a context $\mb u \in Dom(\mb U)$, denoted $Y_{\mb X \leftarrow\mb x}(\mb u)$, is the {\em solution} to $Y$ in the modified set of structural equations.  Potential outcomes satisfy the following {\em consistency rule} used in derivations presented in Section~\ref{sec:identif}.
%
%
\begin{align}
\mb X(\mb u)= \mb x \implies Y_{\mb X \leftarrow\mb x}(\mb u)= y \label{eq:consistency}
\end{align}
This rule states that in contexts where $\mb X=\mb x$, the outcome is invariant to the intervention $\mb X \leftarrow\mb x$. For example, changing the income-level of applicants to high does not change the loan decisions for those who already had high-incomes before the intervention.

The distribution  $\pr(\mb u)$ induces a probability distribution over endogenous variables and potential outcomes. Using PCMs, one can express {\em counterfactual queries} of the form $\pr(Y_{\mb X \leftarrow\mb x}=y \mid \mb k)$, or simply $\pr(y_{\mb X \leftarrow\mb x} \mid \mb k)$; this reads as ``For contexts  with attributes $\mb k$, what is the probability that we would observe $Y=y$ had $\mb X$ been $\mb x$?" and is given by the following expression: 
\vspace{-0.1cm}
\begin{align} 
  \pr(y_{\mb X \leftarrow \mb x} \mid \mb k)
  &= \sum_{\mb u } \ \pr(y_{\mb X \leftarrow \mb x}(\mb u)) \ \pr(\mb u \mid  \mb k) \label{eq:abduction}
\end{align}
%
%
%

\noindent Equation~\eqref{eq:abduction} readily suggests Pearl's three-step procedure for answering counterfactual queries~ \cite{pearl2009causality}[Chapter~7]: (1) update $\pr(\mb u)$ to obtain $\pr(\mb u \mid  \mb k)$ ({\em abduction}), (2) modify the causal model to reflect the intervention ${\mb X \leftarrow \mb x}$ ({\em action}), and (3) evaluate the RHS of (\ref{eq:abduction}) using the index function $\pr(Y_{\mb X \leftarrow \mb x}(\mb u)=y)$ {(\em prediction)}. However, performing this procedure requires the underlying PCM to be fully observed, i.e, the distribution $\pr(\mb u)$ and the underlying structural equations must be known, which is an impractical requirement. In this paper, we assume that only background knowledge of the underlying causal diagram is available, but exogenous variables and structural equations are unknown.

\par {\bf The $\Do$-operator.} For causal diagrams, Pearl defined the $\Do$-operator as a graphical operation that gives semantics to {\em interventional queries} of the form ``What is the probability that we would observe $Y=y$ (at population-level) had $\mb X$ been $\mb x$?", denoted $\pr(\mb y \mid \Do(\mb x))$.  Further, he proved a set of necessary and sufficient conditions under which interventional queries can be answered using historical data.  A sufficient condition is the backdoor-criterion,~\footnote{Since it is not needed in our work, we do not discuss the graph-theoretic notion of backdoor-criterion.} which states that if there exists a set of variables $\mb C$ that satisfy a graphical condition relative to $\mb X$ and $Y$ in the causal diagram  $G$, the following holds (see~\cite{pearl2009causality}[Chapter~3] for details):
\begin{align}\footnotesize
 \pr(\mb y \mid \Do(\mb x))
  &= \sum_{\mb c \in Dom(\mb C)} \pr(\mb y \mid \mb c, \mb x) \ \pr(\mb c)  \label{eq:backdoor}
\end{align}
In contrast to~\eqref{eq:abduction}, notice that the RHS of \eqref{eq:backdoor} is expressed in terms of observed probabilities and can be estimated from historical data using existing statistical and ML algorithms.  

\par {\bf Counterfactuals vs. interventional queries.} The $\Do$-operator is a population-level operator, meaning it can only express queries about the effect of an intervention at population level; in contrast, counterfactuals can express queries about the effect of an intervention on a sub-population or an individual. Therefore, every interventional query can be expressed in terms of counterfactuals, but not vice versa (see~\cite{pearl2016causal}[Chapter~4] for more details).  For instance,   $\pr(y \mid \Do(\mb x))= \pr( y_{\mb X \leftarrow \mb x})$; however, the counterfactual query $\pr( y_{\mb X \leftarrow \mb x}\mid \mb x',  y')$, which asks about the effect of the intervention $\mb X \leftarrow \mb x$ on a sub-population with attributes $\mb x'$ and $y'$, cannot be expressed in terms of the  $\Do$-operator (see Example~\ref{ex:exp2} below). Note that the probabilistic contrastive counterfactual statements in ~\eqref{eq:ppc}, used throughout this paper to explain a black-box decision-making system concerned with the effect of interventions at sub-population and individual levels, cannot be expressed using the $\Do$-operator and therefore cannot be assessed in general when the underlying probabilistic causal models are not fully observed. Nevertheless, in Section~\ref{sec:identif} we establish conditions under which these counterfactuals can be estimated or bounded using data. 




\begin{figure}
	\centering	\includegraphics[scale=0.44]{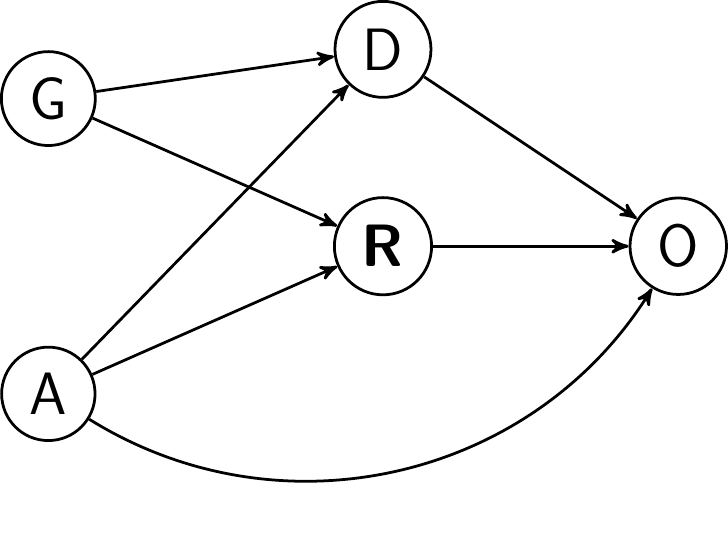}
	\caption{A causal diagram for Example~\ref{ex:exp}.}
	\vspace{-4mm}
	\label{fig:german-dag}
\end{figure}
\begin{example} \label{ex:exp2}
Continuing Example~\ref{ex:exp}, Figure~\ref{fig:german-dag} represents a simple causal diagram for the loan application domain, where $G$ corresponds to the attribute gender, $A$ to age, $D$ to the repayment duration in months, $O$ to the decision of a loan application, and $R$ compactly represents the rest of the attributes, e.g., status of checking account, employment, savings, etc. Note that the loan decision is binary: $O=1$ and $O=0$ indicate whether the loan has been approved or not, respectively. The interventional query $\pr(O=1 \mid \Do(D=24 \text{ months}))$ that is equivalent to the counterfactual $\pr(O_{D \leftarrow 24 \text{ months}} =1)$ reads as ``What is the probability of loan approval at population-level had all applicants selected repayment duration of 24 months?" This query can be answered using data and the causal diagram (since $\{G, A\}$ satisfies the backdoor-criterion in the causal diagram in Figure~\ref{fig:german-dag}).  However, the counterfactual query $\pr(O_{D \leftarrow 24\text{ months}}=1 \mid O=0)$, which reads as `What is the probability of loan approval for a group of applicants whose loan applications were denied had they selected a repayment duration of 24 months?", cannot be expressed using the $\Do$-operator.
\end{example}

\section{Explanations and Recourse Using Probabilistic Counterfactuals}
\label{sec:scores}
In this section, we introduce three measures to quantify the influence of an attribute on decisions made by an algorithm (Section~\ref{sec:probcauses}). We then use these measures to generate different types of explanations for algorithmic systems (Section~\ref{sec:sysexp}).


\subsection{Explanation Scores}
\label{sec:probcauses}

We are given a decision-making algorithm $f : Dom(\mb I) \rightarrow Dom(O)$, where $\mb I$ is set of input attributes (a.k.a. features for ML algorithms) and $O$ is a binary  attribute, where $O=o$ denotes the positive decision (loan approved) and $O=o'$ denotes the negative decision (loan denied). Let us assume we are given a PCM $\pcm$ with a corresponding causal diagram $G$ (this assumption will be relaxed in Section~\ref{sec:identif}) such that
$\mb I \subseteq \mb V$, i.e., the inputs of $f$ are a subset of the observed attributes. 
 Consider an attribute $X \in \mb V$ and a pair of attribute values $x, x' \in Dom(X)$. We quantify the influence of the attribute value $x$ relative to a baseline $x'$ on decisions made by an algorithm using the following scores, herein referred to as {\em explanation scores}; (we implicitly assume an order $x > x'$).  

%
\begin{definition}[Explanation Scores]
  Given a PCM $\pcm$ and an algorithm $f : Dom(\mb X) \rightarrow Dom(O)$, a variable $X \in \mb V$, and a pair of attribute values $x, x' \in Dom(X)$, we quantify the influence of $x$ relative to $x'$ on the algorithm's decisions in the context $\mb k \in Dom(\mb K)$, where $\mb K \subseteq \mb V-\{X,O\}$, using the following  measures:
  \begin{itemize}
  \item The {\em necessity score}:
{    \begin{align} 
       \nec_{x}^{x'}(\mb k) &\defeq \pr(o'_{X \leftarrow x'} \mid x,  o, \mb k) \label{eq:ns} 
    \end{align}
    }
  \item The {\em sufficiency score}:
  {   \begin{align} 
      \suf_x^{x'}(\mb k) & \defeq \pr(o_{X\leftarrow x} \mid x', o', \mb k) \label{eq:ss}
    \end{align}}
  \item The {\em necessity and sufficiency score}:
     {  \begin{align}  
    \nsuf_x^{x'}(\mb k) &\defeq \pr(o_{X\leftarrow x}, o'_{X\leftarrow x'} \mid \mb k) \label{eq:pns},
    \end{align}}
  \end{itemize}
\noindent where the distribution $\pr(o_{X\leftarrow x})$ is well-defined and can be computed from the algorithm $f(\mb I)$.\footnote{For deterministic  $f(\mb I)$, $\pr(o_{X\leftarrow x})= \sum_{\mb i \in Dom(\mb I)}\mathbbm{1}_{\{f(\mb i)=o\}} \ \pr(\mb I_{X\leftarrow x}=\mb i)$, where $\mathbbm{1}_{\{f(\mb i)=o\}}$ is an indicator function.}
\label{def:exp}
\end{definition}
%
For simplicity of notation, we drop $x'$ from $\nec_x^{x'}, \suf_x^{x'}$ and $\nsuf_x^{x'}$ whenever it is clear from the context.
The necessity score in~\eqref{eq:ns} formalizes the probabilistic contrastive counterfactual in~\eqref{eq:ppc}, where \textsf{$<$actual-value$>$} and \textsf{$<$counterfactual-value$>$} are respectively $\mb k  \ \cup \ x$ and $\mb k \ \cup \ x'$, and \textsf{$<$actual-decision$>$} and \textsf{$<$foil-decision$>$} are respectively positive decision $o$ and negative decision $o'$. This reads as ``What is the probability that for individuals with attributes $\mb k$, the algorithm's decision would be {\em negative} instead of {\em positive} had $X$ been $x'$ instead of $x$?"  In other words, $\nec_X(.)$ measures the algorithm's percentage of positive decisions that are {\em attributable to} or {\em due to} the attribute value $x$. The sufficiency score in ~\eqref{eq:ss} is the dual of the necessity score; it reads as ``What would be the probability that for individuals with attributes $\mb k$, the algorithm's decision would be {\em positive} instead of {\em negative}  had $X$ been $x$ instead of $x'$?" Finally, the necessity and sufficiency score in~\eqref{eq:pns} establishes a balance between necessary and sufficiency; it measures the probability that the algorithm responds in both ways. Hence, it can be used to measure the general explanatory power of an attribute.\ignore{\footnote{One can further derive scores for measuring the extent to which an attribute value is necessary but not sufficient, not necessary but sufficient, and neither necessary nor sufficient.}}   In Section~\ref{sec:identif}, we show that the necessary and sufficiency score is non-zero iff $X$ causally influences the algorithm's decisions. (Note that the explanation scores are well-defined for a set of attributes.)

\begin{remark}

\reva{A major difference between our proposal and existing methods in XAI is the ability to account for the indirect influence of attributes that may not be explicitly used in an algorithm's decision making process, but implicitly influence its decisions via their proxies. The ability to account for such influences is particularly important in auditing algorithms for fairness, where typically sensitive attributes, such as race or gender, are not explicitly used as input to algorithms.
For instance, in ~\cite{valentino2012websites} Wall Street Journal investigators reported that a seemingly innocent online pricing algorithm that simply adjusts online prices based on users' proximity to competitors’ stores is discriminative against lower-income individuals. In this case, the algorithm does not explicitly use income; however, it turns out that living further from competitors’ stores is a proxy for low income.~\footnote{\reva{In contrast to mediational analysis in causal inference that studies direct and indirect causal effects~\cite{DBLP:conf/uai/Pearl01,pearl2019seven}, in this paper we are interested in quantifying the sufficiency and necessity scores of attributes explicitly and implicitly used by the algorithm.}}}
\end{remark}

\subsection{\sys's Explanations}
\label{sec:sysexp}
Based on the explanations scores proposed in Section~\ref{sec:probcauses}, \sys\ generates the following types of explanations. 
\par
{\bf  Global, local and contextual explanations.} 
To understand the influence of each variable $X \in \mb V$ on an algorithm's decision, \sys\ computes the necessity score $\nec_x(\mb k)$, sufficiency score $\suf_x(\mb k)$, and necessity and sufficiency score $\nsuf_x(\mb k)$ for each value $x\in Dom(X)$ in the following contexts:  
(1) $\mb K=\emptyset$: the scores measure the {\em global} influence of $X$ on the algorithm's decision. (2) $\mb K=\mb V$: the scores measure the  individual-level or {\em local} influence of $X$ on the algorithm's decision. (3) A user-defined $\mb {K=k}$ with $\emptyset \subsetneq \mb K \subsetneq \mb V$: the scores measure the {\em contextual} influence of  $X$ on the algorithm's decision. In the context 
$\mb k$, \sys\ calculates the explanation scores for an attribute $X$ by computing the maximum score over all pairs of attribute values $x,x'\in Dom(X)$.
In addition to singleton variables, \sys{} can calculate explanation scores for any user-defined set of attributes.


For a given individual, \sys{} estimates the positive and negative contributions of a specific attribute value toward the outcome. Consider an individual with a negative outcome $O=o'$ having the attribute $X=x'$. The negative contribution of $x'$ is characterized by the probability of getting a positive outcome on intervening $X\leftarrow x$, $\ignore{\max\limits_{x>x'} \Pr[o_{X\leftarrow x}| x',o',\mb k] =} \max\limits_{x>x'}\suf_{x}^{x'}(\mb k)$, and 
the positive contribution of $x'$ for the individual is calculated as $\ignore{\max\limits_{x''<x'} \Pr[o_{X\leftarrow x'}| x'',o',\mb k]=}\max\limits_{x''<x'}\suf_{x'}^{x''}(\mb k)$. 
Similarly, for an individual with a positive outcome $O=o$  attribute value $X=x'$, the positive contribution of $x'$ is calculated by estimating the probability of $o'$ if the attribute value was intervened to be smaller than $x'$, $\ignore{\max\limits_{x''<x'}\Pr[o_{X\leftarrow x''}'| x',o,\mb k] =} \max\limits_{x''<x'}\nec_{x'}^{x''}(\mb k)$ 
and the negative contribution of $X=x'$ is $\ignore{\max\limits_{x'<x}\Pr[o_{X\leftarrow x'}'| x,o,k] =} \max\limits_{x>x'} \nec_{x}^{x'}(\mb k)$. Note that the negative contribution of attribute $X=x'$ is calculated by intervening on the individual at hand, but the positive contribution is estimated by intervening on individuals with $X=x''$ to satisfy the same context $\mb k$. In Figure~\ref{fig:ravan}, low credit amount contributes negatively 
to the outcome for \negativeuser\ as increasing credit amount improves 
their chances of getting the loan approved. Attributes like credit history contribute both positively and negatively: poor credit history worsens the chances of approval, but improving credit history furthers the chances of better credit.

\ignore{\paragraph*{\bf Top-$k$ heterogeneous context} In general, the global and contextual influence of a variable may differ or even contradict.  To discover {\em heterogeneity} in influence
of an attribute on an algorithm's decision across different contexts, \sys\ automatically identifies to-$k$ contexts in which the global and contextual influences are significantly different.  This problem is computationally infeasible, however, in section~\ref{sec:opti}, we develop algorithms that effectively identify semantically meaningful contexts that exhibit heterogeneity.}

\par
{\bf Counterfactual recourse.} For individuals for whom an algorithm's decision is negative, \sys\ generates explanations in terms of minimal interventions on a user-specified set of actionable variables $\mb A \subseteq \mb V$ that have a high {\em sufficiency score}, i.e., the intervention can produce the positive decision with high probability.  The explanations can be used either as justification in case the decision is challenged or as a feasible action that the individual may perform in order to improve the outcome in the future (``recourse'').
For example, in Figure~\ref{fig:ravan}, the set of actionable items for \negativeuser\ may consist of her credit amount, loan duration, savings and purpose. Examples of specific actions include ``increase the loan repayment duration'' or ``raise the amount in savings.''


Given an individual with attributes $\mb v$, a set of actionable variables $\mb A \subseteq \mb V$, and a cost function $\text{Cost}\left(\mb a,\hat{\mb a}\right)$ that determines the cost of an intervention that changes $\mb A$ from its current value $\mb a$ to $\hat{\mb a}$, for $\hat{\mb a} \in Dom(\mb A)$, a {\em counterfactual recourse} can be computed using the following
optimization problem:
\vspace{-.1cm}
\begin{align} 
      \argmin_{{\mb a} \in Dom(\mb A)}\text{Cost}\left(\mb a , \hat{\mb a} \right)
  &&& \text{s.t. } \suf_{ \hat{\mb a}}(\mb v) \geq \alpha   \label{eq:opt}
\end{align} 
The optimization problem in \eqref{eq:opt} treats the decision-making algorithm as a black box; hence, it can be solved merely using {\em historical data} (see Section~\ref{sec:computingrecourse}). The solutions to this problem provide end-users with informative, feasible and actionable explanations and recourse by answering questions such as ``What are the best courses of action that, if performed in the real world, would with high probability change the outcome  for this individual?''


\vspace{-.2cm}
\section{Properties and Algorithms}
\label{sec:boundandalg}

In this section, we study properties of the explanation scores in~Section~\ref{sec:scores} and establish conditions under which they can be bounded or estimated from historical data (Section~\ref{sec:identif}). We then  develop an algorithm for solving the optimization problem for computing counterfactual recourse (Section~\ref{sec:computingrecourse}). 
\label{sec:opti}
\subsection{Computing Explanation Scores}
\label{sec:identif}
 Recall from Section~\ref{sec:back} that if the underlying PCM is fully specified, i.e., the structural equations and the exogenous variables are observed, then counterfactual queries, and hence the explanation scores, can be computed via Equation~\eqref{eq:abduction}. However, in many applications, PCMs are not fully observed, and one must estimate explanation scores from  data. First, we prove the following bounds on  explanation scores, computed for a set of attributes $\mb X$.
%
\begin{prop}\label{prop:iden:mon:bounds}  \em
 Given a PCM $\pcm$ with a corresponding causal DAG $\cg$, an algorithm $f : Dom(\mb I) \rightarrow Dom(O)$, and a set of attributes $\mb X\subseteq \mb V-\{O\}$ with two sets of attribute values ${\mb x, \mb x'} \in Dom(\mb X)$, if $\mb K$  consists of non-descendants of $\mb I$ in $\cg$, then the explanation score can be bounded as follows: 
{\small
 \begin{align}
 \max{\left(0, \frac{\splitfrac{{\pr(o, \mb x \mid \mb k)}{+ \pr(o, \mb x'\mid \mb k)}}{- \pr(o \mid \Do(\mb x'), \mb k)}}{\pr(o, \mb x \mid \mb k)}\right)}   \leq  &\nec_{\mb x}(\mb k)  \leq \ \min\left(\frac{\splitfrac{\pr(o' \mid \Do(\mb x'), \mb k)}{- \pr(o', \mb x' \mid \mb k)}}{\pr(o, \mb x \mid \mb k)}, 1\right) 
\label{eq:nbound}  \\
  \max{\left(0,\frac{\splitfrac{{\pr(o', \mb x \mid \mb k)}{ + \pr(o',\mb x' \mid \mb k)}}{- \pr(o' \mid \Do(\mb x), \mb k)}}{\pr(o', x' \mid \mb k)}\right)} \ \leq  &\suf_{\mb x}(\mb k)  \leq  \min\left(\frac{\splitfrac{\pr(o\mid \Do(\mb x), \mb k)} { -  \pr(o,x \mid \mb k)}}{\pr( o', x' \mid \mb k)}, 1\right)
 \label{eq:sbound}\\
      \max\Big(0, \splitfrac{\pr(o \mid \Do(\mb x),  \mb k)}{-\pr(o \mid \Do(\mb x'), \mb k)}   \Big)  \leq &      \nsuf_{\mb x}(\mb k)  \leq \min
\Big(\splitfrac{\pr(o \mid \Do(\mb x), \mb k),}{ \pr(o' \mid \Do(\mb x'), \mb k)}\Big) 
   \label{eq:nsbound} 
 \end{align}
 }
 
 \ignore{
 \hspace{-1cm} {\small
 \begin{align}
\max{\left(0, \frac{\pr(o, \mb x \mid \mb k)+\pr(o, \mb x'\mid \mb k)- \pr(o \mid \Do(\mb x'), \mb k)}{\pr(o, \mb x \mid \mb k)}\right)}   & \leq  \nec_{\mb x}(\mb k)  \leq \ \min\left(\frac{\pr(o' \mid \Do(\mb x'), \mb k)- \pr(o', \mb x' \mid \mb k)}{\pr(o, \mb x \mid \mb k)}, 1\right) 
\label{eq:nbound}  \\
\max{\left(0,\frac{\pr(o', \mb x \mid \mb k) + \pr(o',\mb x' \mid \mb k)- \pr(o' \mid \Do(\mb x), \mb k)}{\pr(o', x' \mid \mb k)}\right)} & \leq  \suf_{\mb X}(\mb k)  \leq  \min\left(\frac{\pr(o\mid \Do(\mb x), \mb k)\big)  -  \pr(o,x \mid \mb k)}{\pr( o', x' \mid \mb k)}, 1\right)
 \label{eq:sbound}\\
     \max\Big(0, \pr(o \mid \Do(\mb x),  \mb k)-\pr(o \mid \Do(\mb x'), \mb k)   \Big)  & \leq      \nsuf_{\mb X}(\mb k)  \leq \min
\Big(\pr(o \mid \Do(\mb x), \mb k) , \pr(o' \mid \Do(\mb x'), \mb k)\Big) 
   \label{eq:nsbound} 
 \end{align}
 }
 }
\end{prop}  
\begin{proof}
We prove the bounds for~(\ref{eq:nbound});~(\ref{eq:sbound}) and~(\ref{eq:nsbound}) are proved similarly. The following equations are obtained from the law of total probability:
{\small
\begin{align}
    \pr(o'_{\mb X\leftarrow \mb x}, \mb x,\mb k)
    &=  \pr(o'_{\mb X\leftarrow \mb x}, o'_{\mb X\leftarrow \mb x'}, \mb x,\mb k)
      +  \pr(o'_{\mb X\leftarrow \mb x}, o_{\mb X\leftarrow \mb x'}, \mb x,\mb k) \label{eq:prop:first}\\
  \pr(o'_{\mb X\leftarrow \mb x'}, \mb x,\mb k)
    &=  \pr(o'_{\mb X\leftarrow \mb x'}, o'_{\mb X\leftarrow \mb x}, \mb x,\mb k)
      +  \pr(o'_{\mb X\leftarrow \mb x'}, o_{\mb X\leftarrow \mb x}, \mb x,\mb k)       \label{eq:prop:second2} \\
           \pr(o'_{\mb X\leftarrow \mb x'}, \mb k) &=      \pr(o'_{\mb X\leftarrow \mb x'}, \mb x, \mb k) +      \pr(o'_{\mb X\leftarrow \mb x'}, \mb x', \mb k) + \nonumber  \sum_{x'' \in Dom(X)-\{x,x'\}}\pr(o'_{\mb X\leftarrow \mb x'}, \mb x'', \mb k)
      \label{eq:prop:second3}\\
  \end{align}
}
By rearranging \eqref{eq:prop:first} and \eqref{eq:prop:second2}, we obtain the following equality:
{\small 
  \begin{align}
\pr(o'_{\mb X\leftarrow \mb x'}, o_{\mb X\leftarrow \mb x}, \mb x,\mb k) &=
     \pr(o'_{\mb X\leftarrow \mb x'}, \mb x,\mb k) -  \pr(o'_{\mb X\leftarrow \mb x}, \mb x,\mb k)+   \pr(o'_{\mb X\leftarrow \mb x}, o_{\mb X\leftarrow \mb x'}, \mb x,\mb k)   \label{proof:bound:4}
  \end{align}
}  
The following bounds for the LHS of \eqref{proof:bound:4} are obtained from the Fréchet bound.\footnote{$ \max \big(0,\sum_{x \in \mb X}\pr(x)-(|\mb x|-1) \big) \leq \pr(\mb x)\leq \min_{x \in \mb x}\pr(x)$}
{\small
\begin{align}
LHS& \geq \pr(o'_{\mb X\leftarrow \mb x'}, \mb k)- \pr(o', \mb x', \mb k)  -   \pr(o'_{\mb X\leftarrow \mb x}, \mb x,\mb k) -    \sum_{x'' \in Dom(\mb X)-\{\mb x,\mb x'\}}\pr(o'_{\mb X\leftarrow \mb x'}, \mb x'', \mb k)  \nonumber  \\ &  \hspace{1cm} \texttt{(obtained from  Eq.~\eqref{eq:prop:second3} and ~\eqref{eq:consistency}, lower bounding  $\pr(o'_{\mb X\leftarrow \mb x}, o_{\mb X\leftarrow \mb x'}, \mb x,\mb k)$)}
\nonumber \\
      & \geq  \pr(o'_{\mb X\leftarrow \mb x'}, \mb k) - \pr(o', x,\mb k)-\pr(o', x',\mb k)  - \pr(\mb k) +\pr(x, \mb k)+\pr(x',\mb k) \nonumber \\ & =        \pr(o'_{\mb X\leftarrow \mb x'}, \mb k) + \pr(o, x,\mb k)+\pr(o, x',\mb k)  - \pr(\mb k) \nonumber \\ & =
        \pr(o, x,\mb k)+\pr(o, x',\mb k)- \pr(o_{\mb X\leftarrow \mb x'}, \mb k)  \label{proof:bounds:1} \\
LHS &  \leq   \pr(o'_{\mb X\leftarrow \mb x'}, \mb x,\mb k) -  \pr(o'_{\mb X\leftarrow \mb x}, \mb x,\mb k)+   \pr(o_{\mb X\leftarrow \mb x'}, \mb x,\mb k) = \pr(o, \mb x,\mb k)  \label{proof:bounds:2}   \\ &  \hspace{1.8cm} \texttt{(obtained by upper bounding $\pr(o'_{\mb X\leftarrow \mb x}, o_{\mb X\leftarrow \mb x'}, \mb x,\mb k)$  in Eq.~\eqref{proof:bound:4}} \nonumber \\
LHS & \leq \pr(o'_{\mb X\leftarrow \mb x'}, \mb k)- \pr(o', \mb x', \mb k) - \sum_{x'' \in Dom(X)-\{x,x'\}}\pr(o'_{\mb X\leftarrow \mb x'}, \mb x'', \mb k) \nonumber \\ & \hspace{1cm} \texttt{(obtained from  Eq.~\eqref{eq:prop:second3} and ~\eqref{eq:consistency}, upper bounding $\pr(o'_{\mb X\leftarrow \mb x}, o_{\mb X\leftarrow \mb x'}, \mb x,\mb k)$)} \nonumber \\
& \leq \pr(o'_{\mb X\leftarrow \mb x'}, \mb k)- \pr(o', \mb x', \mb k) \label{proof:bounds:3} 
  \end{align}
 }
\noindent Equation~\eqref{eq:nbound} is obtained by dividing \eqref{proof:bounds:1},
\eqref{proof:bounds:2} and \eqref{proof:bounds:3} by ${\pr(o, \mb x,\mb k) }$, applying the consistency rule~\eqref{eq:consistency}, and considering the fact that since $\mb K$ consists of non-descendants of $\mb X$, the intervention $\mb X\leftarrow \mb x'$ does not change $\mb K$; hence,  $\pr(o_{\mb X\leftarrow \mb x'} \mid \mb k)=\pr(o \mid \Do(\mb x'), \mb k)$.\\
\end{proof}

%
%

Proposition~\ref{prop:iden:mon:bounds} shows the explanation scores can be bounded whenever interventional queries of the form $\pr(o \mid \Do(\mb x), \mb k)$ can be estimated from historical data using the underlying causal diagram $G$~(cf.~Section~\ref{sec:back}).  The next proposition further shows that if the algorithm is {\em monotone} relative to  $\mb x,  \mb x' \in Dom(\mb X)$, i.e., if $\mb x > \mb x'$, then $O_{\mb X \leftarrow \mb x } \geq O_{\mb X \leftarrow \mb x' }$\footnote{Monotonicity expresses the assumption that changing $\mb X$ from $\mb x'$ to $\mb x$ cannot change the algorithm’s decision from positive to negative; increasing $\mb X$ always helps.}, and the exact value of the explanation scores can be computed from data. \revc{(In case the ordering between $\mb x$ and $\mb x'$ is not known apriori (e.g., for categorical values), we infer it by comparing the output of the algorithm for $\mb x$ and $\mb x'$.)}

\vspace{-0.1cm}
\begin{prop} 
 \label{prop:iden:mon:ci} Given a causal diagram $G$, if the decision-making algorithm $f : Dom(\mb I) \rightarrow Dom(O)$ is monotone relative to $\mb x,  \mb x' \in Dom(\mb X)$ and if there exists a set of variables $\mb C \subseteq \mb V-\{\mb K \cup \mb X\}$ such that $\mb C \cup \mb K$ satisfies the backdoor-criterion relative to $\mb X$ and $\mb I$ in $G$, the following holds:  \em
{\small
 \begin{align}        
 \nec_{\mb x}(\mb k) &= \frac{ \Big(\sum_{c\in Dom(\mb C)} \pr(o' \mid \mb c,\mb x',\mb k) \ \pr(\mb c \mid \mb x, \mb k)\Big)  - \pr(o' \mid \mb x,\mb k) }{\pr(o \mid \mb x,   \mb k)} \label{eq:nec:mon:ident} \\
 \suf_{\mb x}(\mb k)&=  \frac{ \Big(\sum_{\mb c \in Dom(\mb C)} \ \pr(o\mid \mb c, \mb x ,\mb k)  \pr(\mb c\mid \mb x', \mb  k)\Big) - \pr(o\mid \mb x', \mb k) }{\pr(o' \mid  \mb x', \mb k)} \label{eq:suff:mon:ident}  \\
\nsuf_{\mb x }(\mb k) &= \sum_{\mb c \in Dom(\mb C)} \big(\pr(o \mid \mb x, \mb k, \mb c)-\pr(o \mid \mb x', \mb c, \mb k)\big) \ \pr(\mb c\mid \mb k) \label{eq:nec:nsuff:ident} 
 \end{align}   
 }
\end{prop}

\begin{proof} 
{
Here, we only prove \eqref{eq:nec:mon:ident}. The proof of \eqref{eq:suff:mon:ident} and \eqref{eq:nec:nsuff:ident} are similar.
Notice that monotonicity implies $\pr(o'_{\mb X\leftarrow \mb x}, o_{\mb X\leftarrow \mb x'}, \mb x,\mb k)=0$. Also note that if $\mb C \cup \mb K$ satisfies  the backdoor-criterion, then the following independence, known as conditional ignorability, holds: $(O_{\mb X\leftarrow \mb x} \indep \mb X \mid \mb C \cup \mb K)$\cite{pearl2016causal}[Theorem 4.3.1]. We show \eqref{eq:nec:mon:ident} in the following steps:}
{   
  \begin{align}
    \nec_{\mb x}(\mb k)
    &= \frac{\pr(o'_{X \leftarrow \mb x'}, \mb x,  o, \mb k)}{\pr(o, \mb x, \mb k)}= \frac{\pr(o'_{\mb X\leftarrow \mb x'} \mid \mb x,\mb k) - \pr(o' \mid \mb x,\mb k) }{\pr(o \mid \mb x,   \mb k)} \\& \hspace{3.2cm}\texttt{(from Eq.~\eqref{proof:bound:4}, ~\eqref{eq:consistency} and monotonicity)}\nonumber  \\ 
    &= \frac{\sum_{c\in Dom(C)} \pr(o'_{\mb X\leftarrow \mb x'} \mid  \mb c, \mb x,\mb k) \ \pr(\mb c \mid \mb x, \mb k)  - \pr(o' \mid \mb x,\mb k) }{\pr(o \mid \mb x,   \mb k)}\nonumber \\        
    &= \frac{\sum_{c\in Dom(C)} \pr(o' \mid \mb c,\mb x',\mb k) \ \pr(\mb c \mid \mb x, \mb k)  - \pr(o' \mid \mb x,\mb k) }{\pr(o \mid \mb x,   \mb k)} \nonumber \\ & \hspace{2.4cm}  \texttt{(from conditional ignorability and Eq.~\eqref{eq:consistency})} \nonumber
  \end{align}
}
\end{proof}

Therefore, Proposition~\ref{prop:iden:mon:ci}  facilitates bounding and estimating explanation scores from historical data when the underlying probabilistic causal models are not fully specified but background knowledge on the causal diagram is available. (See~ Section~\ref{sec:related} for a discussion about the absence of causal diagrams). 
We establish the following connection between the explanation scores.

  


\begin{prop} \label{pro:exscores:rel} Explanation scores are related through the following inequality. For a binary $X$, the inequality becomes an  equality.  \em
{\small
\begin{align}
\nsuf_{\mb x}(\mb k) \leq \Pr(o, \mb x \mid \mb k) \ \nec_{\mb x}(\mb k)+ \Pr(o', \mb x' \mid \mb k) \ \suf_{\mb x}(\mb k)  + 1-\pr(\mb x\mid \mb k)- \pr(\mb x'\mid \mb k) \label{eq:necsuff}
\end{align} 
}
\end{prop}

{
\begin{proof} The inequality is obtained from 
the law of total probability, the consistency rule in (\ref{eq:consistency}), and applying the Fréchet bound, as shown in the following steps:
 {\small
  \begin{align}
  \nsuf_{\mb  x}(\mb k)  &=\frac{\pr(o_{\mb  X\leftarrow \mb  x}, o'_{X\leftarrow x'}, \mb k)}{\pr(\mb k)} \texttt{(from consistency~\eqref{eq:consistency})}  \nonumber\\ & \hspace{-1cm} = \frac{1}{\pr(\mb k)} \big(
    \sum_{x \in \{\mb  x,\mb  x'\}}
     \pr(o_{\mb  X\leftarrow \mb  x}, o'_{\mb  X\leftarrow \mb  x} , \mb k, x)  + 
     \sum_{x''  \in Dom(X)-\{x,x'\}} 
    \pr(o_{X\leftarrow x}, o'_{X\leftarrow x'} , \mb k, \mb  x'') \big) \nonumber\\ & \hspace{3.5cm}\texttt{(from law of total probability)} \nonumber \\
    & \hspace{-1cm} \leq  \frac{1}{\pr(\mb k)} \big(\pr(o_{X\leftarrow x}, o'_{\mb X\leftarrow \mb x} , \mb k, x)+  \pr(o_{\mb X\leftarrow \mb x}, o'_{\mb X\leftarrow \mb x'} , \mb k, x') + 
    \sum_{\mb  x''  \in Dom(X)-\{\mb  x,\mb x'\}} 
    \pr(\mb k, x'')\big) \nonumber\\  & \hspace{2.9cm}\texttt{(obtained by upper bounding the sum)} \nonumber \\
    & \hspace{-1cm} \leq  \frac{1}{\pr(\mb k)} \big(\pr(o, o'_{\mb  X\leftarrow \mb  x'} , \mb k, \mb  x)+
      \pr(o_{\mb  X\leftarrow \mb x}, o' , \mb k, \mb x')+ \pr(\mb k)-\pr(\mb x, \mb k)-\pr(\mb x', \mb k)\big)\nonumber \\ 
      & \hspace{-1cm}   \leq \Pr(o, \mb x \mid \mb k) \ \nec_{\mb x}(\mb k)+ \Pr(o', \mb x' \mid \mb k) \ \suf_{\mb x}(\mb k)  + 1-\pr(\mb x\mid \mb k)- \pr(\mb x'\mid \mb k)\nonumber
  \end{align}
 }
\end{proof}
}
Therefore, for binary attributes, the necessary and sufficiency score can be seen as the weighted sum of necessary and sufficiency scores. Furthermore, the lower bound for the necessity and sufficiency score in Equations~\eqref{eq:nsbound} is called the (conditional) {\em causal effect} of $X$ on $O$~\cite{pearl2017detecting}. Hence, if the causal effect of $X$ on the algorithm's decision is non-zero, then so is the necessity and sufficiency score (for a binary $X$, it is implied from~\eqref{eq:necsuff} that at least one of the sufficiency and necessity scores must also be non-zero). The following proposition shows the converse. 
\begin{prop}
Given a PCM $\pcm$ with a corresponding causal DAG $\cg$, an algorithm $f : Dom(\mb Z) \rightarrow Dom(O)$ and an attribute $X \in \mb V$, if $O$ is a non-descendant of $X$, i.e., there is no causal path from $X$ to $O$, then for all $(x,x') \in Dom(X)$ and for all contexts $\mb k \in Dom(\mb K)$, where $\mb K \subseteq \mb V-\{X,O\}$ , it holds that {\em $\nec_x(\mb k)=\suf_x(\mb k)=\nsuf_x(\mb k)=0$}.
\end{prop}

{
\begin{proof} Let $X$ be any non-descendant of $O$ in the causal diagram $\cg$. 
Since the potential outcomes are invariant to interventions, the implication from Equation~\eqref{eq:abduction} is that for all $(x,x') \in Dom(X)$ and for any set of attributes, $\mb e$: $\pr(o_{\mb X\leftarrow \mb x'} \mid \mb e)=\pr(o_{\mb X\leftarrow \mb x} \mid \mb  e)$. This equality and consistency rule~\eqref{eq:consistency} implies that $\nec_x(\mb k) = \pr(o'_{X \leftarrow x'} \mid x,  o, \mb k)=\pr(o'_{X \leftarrow x} \mid x,  o, \mb k)=\pr(o' \mid x,  o, \mb k)=0$.  $\nsuf_x(\mb k)=0$ and $\suf_x(\mb k)=0$ can be proved similarly.
\end{proof}
}

\revb{\noindent \textbf{Extensions to multi-class classification and regression.} For multi-valued outcomes, i.e.,  $Dom(O) = \{o_1, \ldots, o_{\gamma}\}$, we assume an ordering of the values  $o_1 >  \ldots > o_{\gamma}$ such that $o_i > o_j$ implies that $o_i$ is more desirable than $o_j$. This assumption holds in tasks where certain outcomes are favored over others and holds for real-valued outcomes that have a natural ordering of values. 
We partition $Dom(O)$ into sets $O^{<}$ and $O^{\geq}$ where $O^{<}$ denotes the set of values less than $o$ and $O^{\geq}$ denotes the set of values greater than $o$. Note that we do not require a strict ordering of the values and can simply partition them as \textit{favorable} and \textit{unfavorable}. In these settings, we redefine the explanation scores with respect to each outcome value $o$. For example, necessity score is defined as the probability that the outcome $O$ changes from a value greater than or equal to $o$ to a value lower than $o$ upon the intervention $X \leftarrow x'$:
{
  \begin{align*} 
    \nec_{x}^{x'}(\mb k, o) &\defeq \pr(O^<_{X \leftarrow x'} \mid x,  O^{\geq}, \mb k)
    \end{align*}
}    
The other two scores can be extended in a similar fashion. Our propositions extend to these settings and can be directly used to evaluate the explanation scores using observational data.
 
}

\subsection{Computing Counterfactual Recourse}
\label{sec:computingrecourse}  
Here, we describe the solution to the optimization problem discussed in~(\ref{eq:opt}) for providing an actionable recourse. We formulate our problem as a combinatorial optimization problem over the domain of actionable variables and express it as an integer programming (IP) problem of the form:
 {\small 
\begin{align}
    \argmin_{\hat{\mb a} \in Dom(\mb A)} \quad \sum_{A \in \mb A} \left( \phi_{A}  \sum_{a \in Dom(A)} \delta_{a} \right)&
    \label{eq:opt1_obj} \\
    \text{subject to } \quad \quad \quad \quad \quad \suf_{\hat{\mb a}}(\mb v) &\geq \alpha\label{eq:opt1_c1}\\
    \sum_{a \in Dom(A)} \delta_{a} &\leq 1, \quad \forall A \in \mb A \label{eq:opt1_c2} \\
    \quad \delta_{a} \in \{0, 1\}, \quad \forall a \in Dom(A), A \in \mb A
\end{align}
}
The objective function in the preceding IP is modeled as a linear function over the cost of actions over individual actionable variables. $\phi_{A}$ is a convex cost function that measures the cost of changing $A=a$ to $ A=\hat{a}$, for each $A \in A$ ($\phi_{A}$ = 0 when no action is taken on $A$) and can be predetermined as $\hat{a}$ deviates from $a$ (e.g., the cost could increase linearly or quadratically with increasing deviation from $A=a$). \revc{Constraint~(\ref{eq:opt1_c1}) ensures that action $\hat{\mb a}$ will result in a sufficiency score greater than the user-defined threshold $\alpha$.} In other words, the intervention $\mb A \leftarrow \hat{\mb a}$ can lead to the positive outcome with a probability of at least $\alpha$. Constraint~(\ref{eq:opt1_c2}) and indicator variables $\delta_a$ ensure that of all values in the domain of an actionable variable, only one is acted upon (or changed). \revb{Note that the IP formulation ensures that multiple actions can be taken at the same time. In particular, when $\delta_a=0$, $\forall a \in Dom(A)$, $\forall A \in \mb A$, it implies no action is taken since~(\ref{eq:opt1_c1}) is already satisfied.} 
To compute the sufficiency score in~(\ref{eq:opt1_c1}) from historical data, we rewrite it as $\suf_{\hat{\mb a}}(\mb k \cup {\mb a}) \geq \alpha$, where $\mb K$ consists of all non-descendants of $\mb A$ in the underlying causal diagram $G$, and we assume that $\mb K$ satisfies the backdoor-criterion relative to $O$ and $\mb A$~(cf. Section~\ref{sec:back}). (See Section~\ref{sec:related} for a discussion about violation of the assumptions.) Then, we can incorporate the {\em lower bound} obtained for the sufficiency score in Proposition~\ref{prop:iden:mon:ci} in the optimization problem, as follows: 
{
\begin{align}
 \pr(o\mid \hat{\mb a} ,\mb k) & \geq \pr(o\mid \mb a, \mb k) + \alpha \ \pr(o'\mid \mb a, \mb k) \label{eq:opt1_c1_corollary}
\end{align} 
}
Since $\mb k, \mb a, \alpha$ are constant, the RHS of~(\ref{eq:opt1_c1_corollary}) is also constant and can be pre-computed from data. We estimate the logit transformation of $\pr(o \mid \hat{\mb a}, \mb k)$ and model it as a linear regression equation. This allows us to express~(\ref{eq:opt1_c1_corollary}) as a linear inequality constraint for the IP in~(\ref{eq:opt1_obj}). If a solution to the IP is found, then an action is performed on each variable for which the indicator variable has a non-zero assignment. {\em The solution to this optimization problem can be seen as a recourse that can change the outcome of the algorithm with high probability for individuals with attributes $\mb k$ for which the algorithm made a negative decision.} Note that {the number of constraints in this formulation grows linearly with the number of actionable variables} (which is usually a much smaller subset of an individual's attributes).

\vspace{-.2cm}
\section{Experiments}
\label{sec:exp}


This section presents experiments that evaluate the effectiveness of \sys. We answer the following questions. \textbf{Q1:} What is the end-to-end performance of \sys\ in terms of gaining insight into black-box machine learning algorithms? \revc{How does the performance of \sys\ change with varying machine learning algorithms?} \textbf{Q2:} How does \sys\ compare to state-of-the-art methods in XAI\ignore{that provide explanations for black-box predictive classification algorithms and generate actionable recourse}? \textbf{Q3:} To what extent are the explanations and recourse options \revb{generated by \sys\ correct? }
\vspace{-.1cm}

{\scriptsize
\begin{table}[] \centering
		\begin{tabular}{@{}lrrrrrrrrr@{}}\toprule
			{Dataset} & {Att. [$\#$]} & {Rows[$\#$]} &  Global & Local &  Recourse   \\ \midrule
			\textbf{Adult}~\cite{Adult} & 14 & 48k & 7.5  &4.2 & 3.7 &  \\ \hdashline
			\textbf{German}~\cite{Dua2019}& 20  & 1k & 0.75 & 0.42&2.24 & \\ \hdashline
			\textbf{COMPAS}~\cite{compas}& 7  & 5.2k & 2.03  & 1.34 & - & \\ \hdashline
			\revb{\textbf{Drug}}~\cite{Dua2019}& \revb{13} &\revb{1886}  & \revb{1.25}  & \revb{0.84} & \revb{-}   \\ \hdashline
			\textbf{German-syn}& 6  & 10k & 1.35  &1.01 & 1.65 & \\ \hdashline
		\end{tabular}
	\caption{\textmd{Runtime in seconds for  experiments in Sec.~\ref{sec:exp:endtoend}.}}
	\vspace{-8mm}
	\label{tbl:data}
\end{table}
}
\subsection{Datasets}
We used the following ML benchmark datasets (also in Table~\ref{tbl:data}):\\
\noindent\textbf{German Credit Data (\textsf{German})~\cite{Dua2019}.} This dataset consists of records of bank account holders with their personal, financial and demographic information. 
The prediction task classifies individuals as good/bad credit risks.\\
\noindent\textbf{Adult Income Data (\textsf{Adult})~\cite{Dua2019}.} This dataset contains demographic information of individuals along with information on their level of education, occupation, working hours etc. The task is to predict whether the annual income of an individual exceeds $50$K.\\
\noindent\textbf{\textsf{COMPAS}~\cite{compas}.} This dataset contains information on offenders from Broward County, Florida.
We consider the task of predicting whether an individual will recommit a crime within two years.
\\
\revb{\noindent\textbf{Drug Consumption (\textsf{Drug})~\cite{Dua2019}.} This dataset contains demographic information and personality traits (e.g., openness, sensation-seeking) of individuals. We consider a multi-class classification task of predicting when an individual consumed magic mushrooms: (i)~never, (ii)~more than a decade ago, and (iii)~in the last decade.}\\
\noindent\textbf{\textsf{German-syn}.} We generate synthetic data following the causal graph of the \texttt{German} dataset. \revb{The black-box algorithm runs random forest regressor to predict a credit score within the range $[0,1]$ where $1$ denotes excellent credit while $0$ denotes the worst credit history.} This dataset is specifically used to evaluate the correctness of \sys{}'s scores compared to ground truth scores calculated using the structural equations.

 
\vspace{-.3cm}
\subsection{Setup}
\revc{We considered four black-box machine learning algorithms: random forest classifier~\cite{sklearn}, random forest regression~\cite{sklearn}, XGBoost~\cite{xgboost}, and a feed forward neural network~\cite{fastai}.}
We used the causal diagrams presented in~\cite{DBLP:conf/aaai/Chiappa19} for the Adult and German datasets and ~\cite{nabi2018fair} for COMPAS. \revb{For the Drug dataset, 
the attributes \textsf{Country, Age, Gender} and \textsf{Ethnicity} are root nodes that affect the outcome and other attributes; the outcome is also affected by the other attributes.} We implemented our scores and the recourse algorithm in Python. We split each dataset into training and test data, learned a black-box algorithm (random forest classifier unless stated otherwise) over training data, and estimated conditional probabilities in~(\ref{eq:nec:mon:ident})-(\ref{eq:nec:nsuff:ident}) by regressing over test data predictions. We report the explanation scores for each dataset under different scenarios.
To present local explanations
, we report the positive and negative contributions of an attribute value toward the current outcome (e.g., in Figure~\ref{fig:exp:rel:local:drug}, bars to the left (right) represent negative (positive) contributions of an attribute value). To recommend recourse to individuals who receive a negative decision, we generate a set of actions with the minimal cost that, if taken, can change the algorithm's decision for the individual in the future with a user-defined probability threshold $\alpha$. 



\ignore{\textbf{Baselines.} We compare the effectiveness of \sys{}'s global explanations with \texttt{SHAP}~\cite{DBLP:conf/nips/LundbergL17} and \texttt{Feat}~\cite{10.1023/A:1010933404324}, local explanations with \texttt{SHAP} and \texttt{LIME}~\cite{ribeiro2016should}, and recourse with \texttt{LinearIP}~\cite{ustun2019actionable}.\footnote{We contacted authors of \cite{karimi2020algorithmic},  but their technique does not work for categorical actionable variables. Therefore, we do not consider that technique in our evaluation. } We used open-source implementations of the respective techniques.}

\vspace{-0.3cm}
\subsection{End-to-End Performance}
\label{sec:exp:endtoend}
\begin{figure*}
      \centering
      \subcaptionbox{\texttt{German} \label{fig:exp:end:global:german}}
        {\includegraphics[width=.24
        \columnwidth]{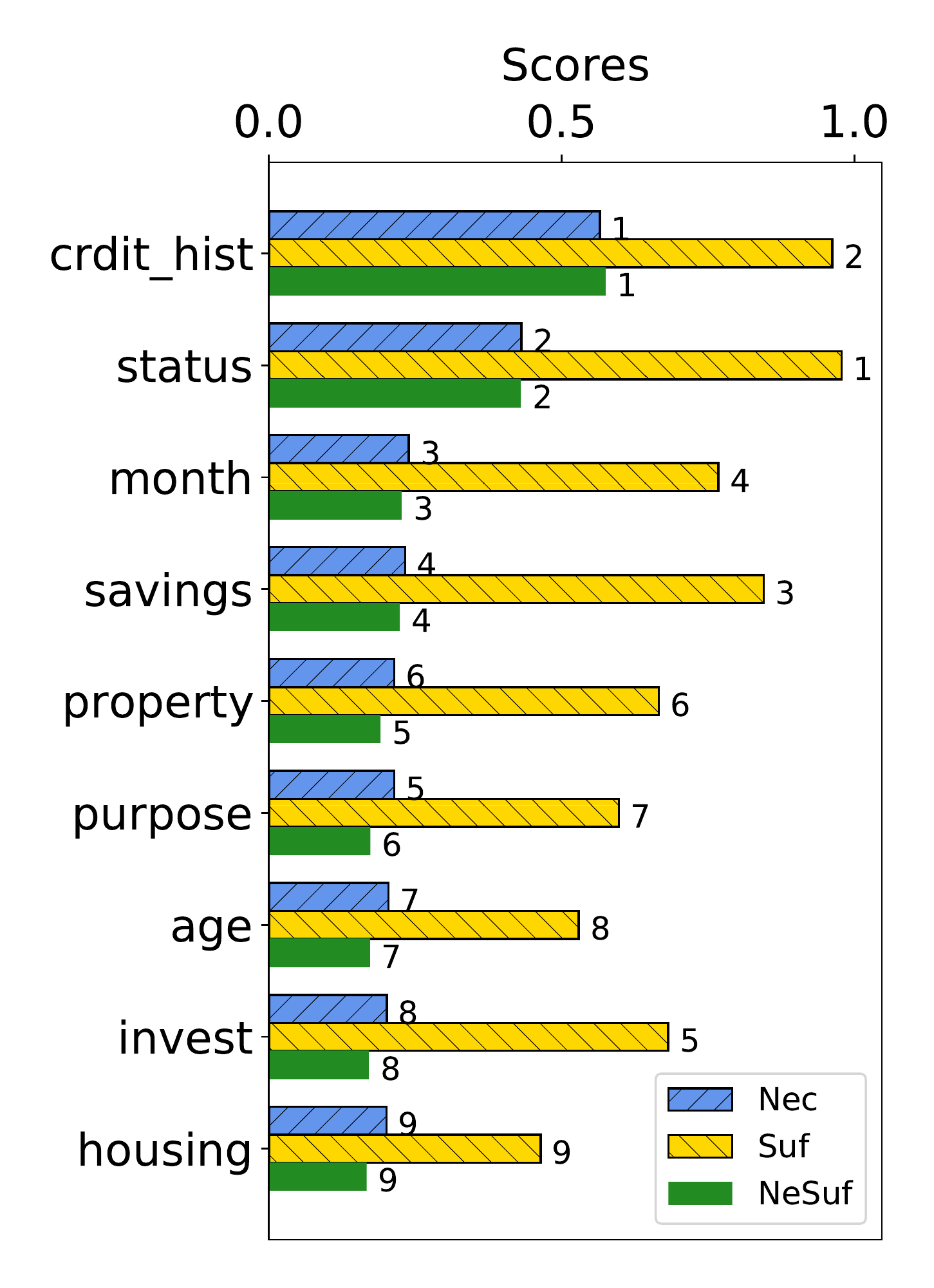}}
      \subcaptionbox{\texttt{Adult} \label{fig:exp:end:global:adult}}
        {\includegraphics[width=.24\columnwidth]{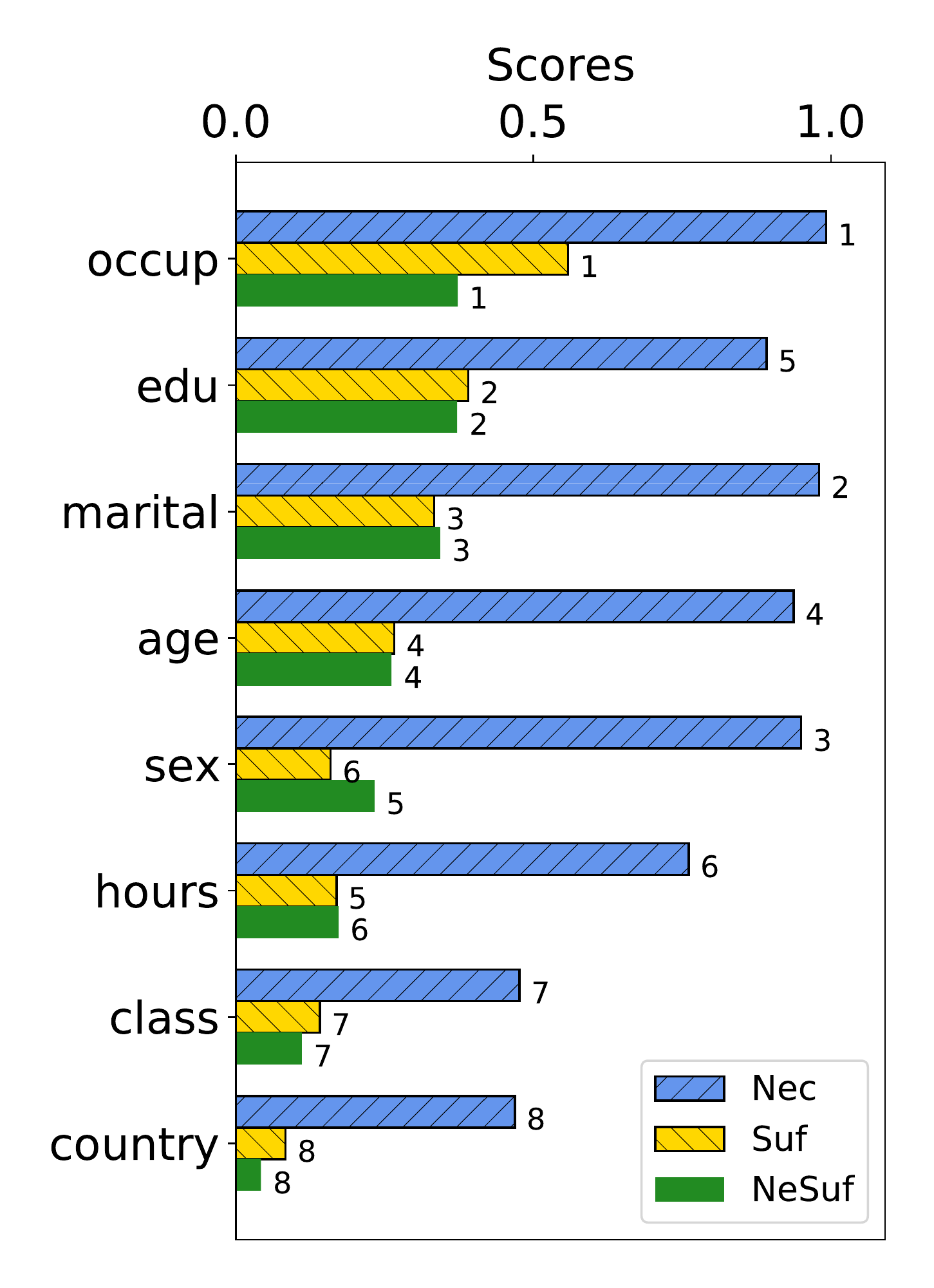}}
         \subcaptionbox{\texttt{Compas} -- Software score \label{fig:exp:end:global:compasscore}}
        {\includegraphics[width=.24\columnwidth]{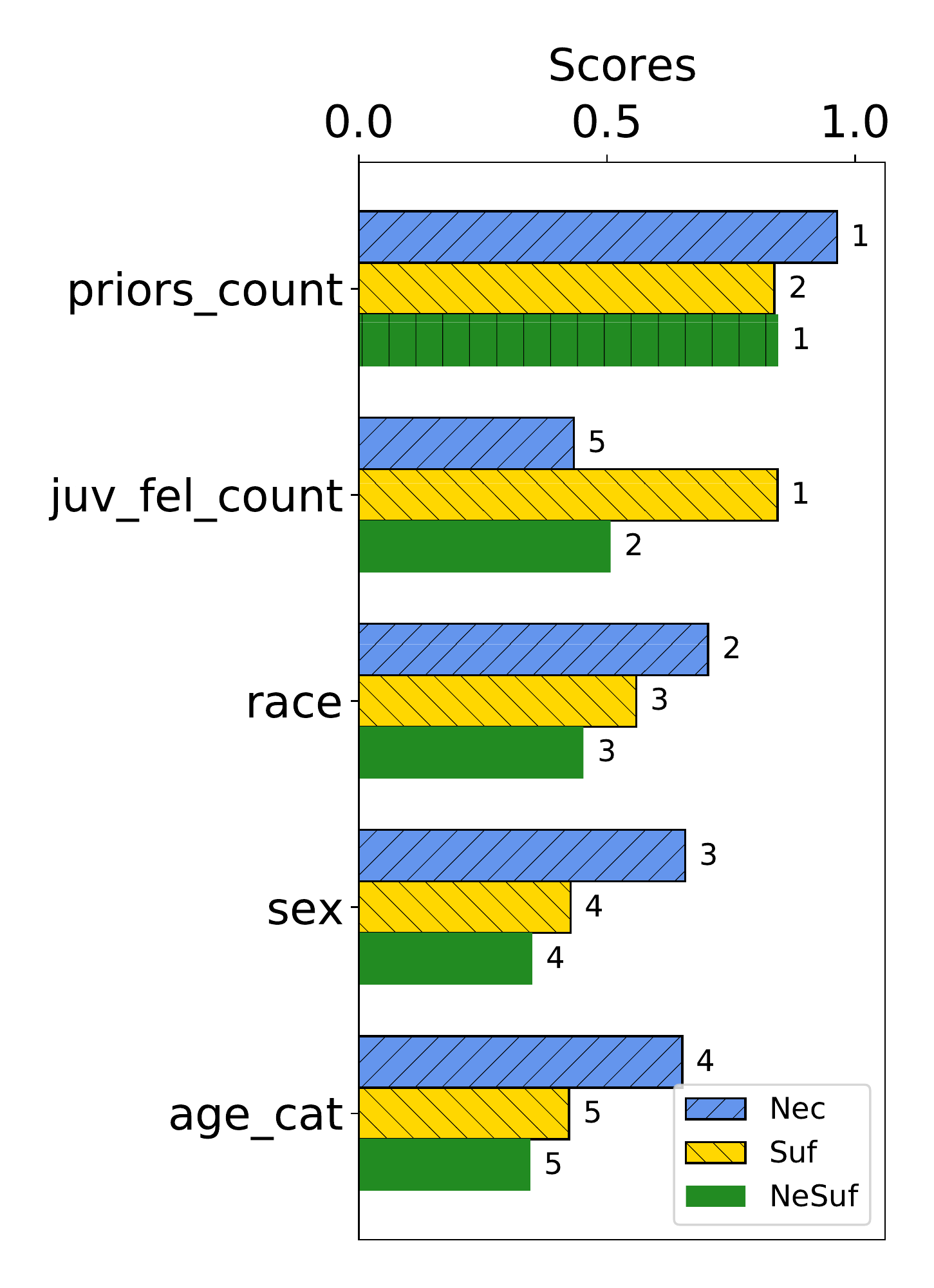}}
        \subcaptionbox{\revb{\texttt{Drug}} \label{fig:exp:end:global:multi-class}}
        {\includegraphics[width=.24
        \columnwidth]{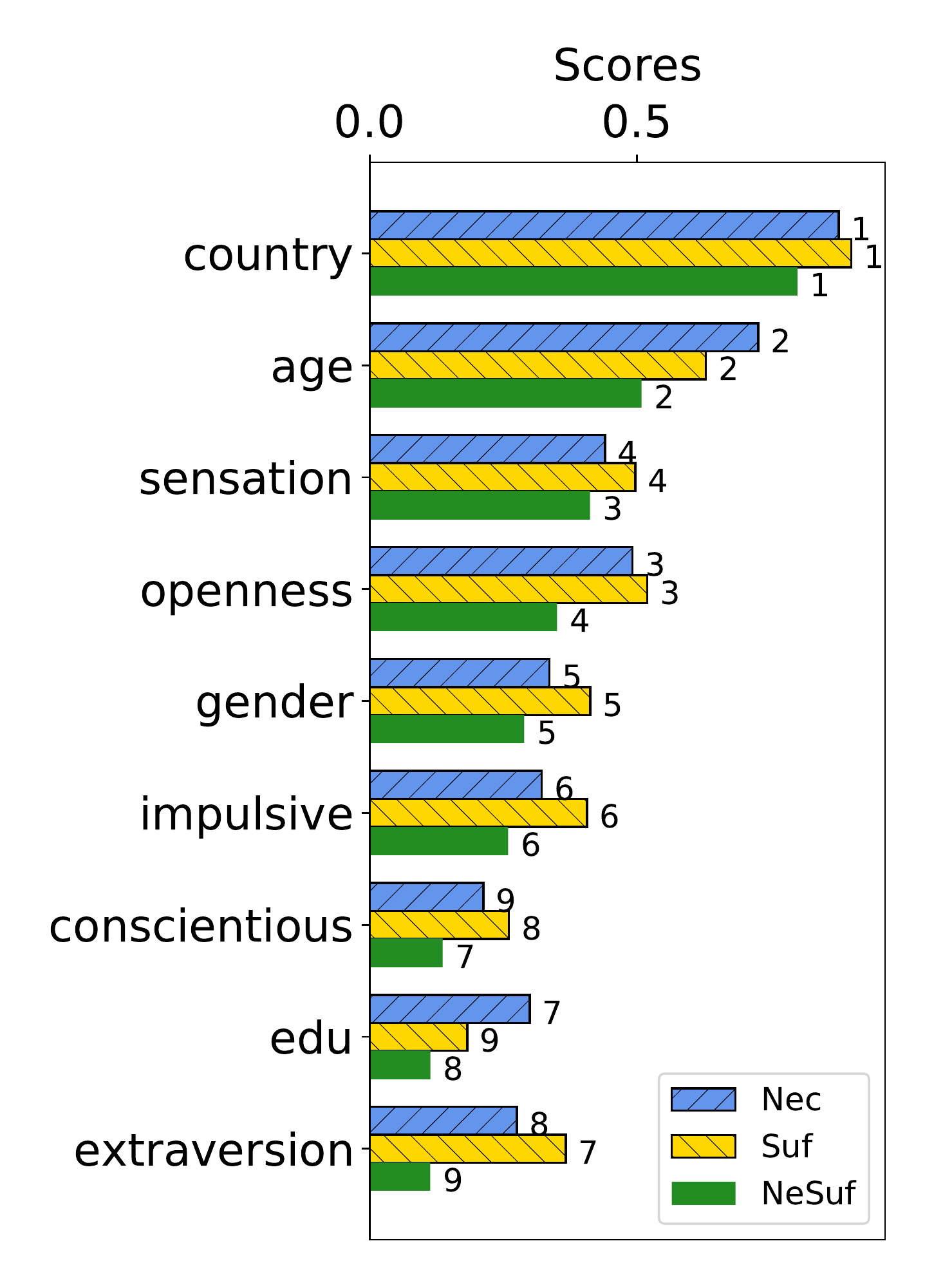}}
      \caption{Global explanations generated by \sys{}, ranking attributes in terms of their necessity, sufficiency and necessity and sufficiency scores. The rankings are consistent with insights from existing literature.}\label{fig:exp:end:global}
\end{figure*}

In the following experiments, we present the local, contextual and global explanations and recourse options generated by \sys.
\ignore{We present results of learning a random forest classifier over training data, and we estimated the conditional probabilities by learning a random forest regressor over test data predictions. Note that our results are independent of the ML algorithms used in these two steps.} In the absence of ground truth, we discuss the coherence of our results with intuitions from existing literature.  Table~\ref{tbl:data} reports the running time of \sys\ for computing explanations and recourse. 

\noindent\textbf{German.} Consider the attributes \texttt{status} and \texttt{credit\_history} in Figure~\ref{fig:exp:end:global:german}. Their near-perfect sufficiency scores indicate their high importance toward a positive outcome at the population level. For individuals for whom the algorithm generated a negative outcome, an increase in their credit history or maintaining above the recommended daily minimum in checking accounts (\texttt{status}) is more likely to result in a positive decision compared to other attributes such as housing or age. These scores and the low necessity scores of attributes are aligned with our intuition about good credit risks: (a)~good credit history and continued good status of checking accounts add to the credibility of individuals in repaying their loans, and (b)~ multiple attributes favor good credit and a reduction in any single attribute is less likely to overturn the decision.

In Figure~\ref{fig:exp:end:contextual:german}, we present contextual explanations that capture the effect of intervening on \texttt{status} in different age groups. We observe that increasing in the status of checking account from $<0$ DM to $>200$ DM is more likely to reverse the algorithm's decision for older individuals (as seen in their higher sufficiency score compared to younger individuals). This behavior can be attributed to the fact that along with checking account status,  loan-granting decisions depend on variables such as credit history, which typically has good standing for the average older individual. For younger individuals, even though the status of their checking accounts may have been raised, their loans might still be rejected due to the lack of a credible credit history.

We report the local explanations generated by \sys\ in Figure~\ref{fig:exp:end:local:german}. In the real world, younger individuals and individuals with inadequate employment experience or insufficient daily minimum amount in checking accounts are less likely to be considered good credit risks. This observation is evidenced in the negative contribution of status, age and employment for the negative outcome example.  For the positive outcome example, current attribute values contribute toward the favorable outcome. Since increasing any of them is unlikely to further improve the outcome, the values do not have a negative contribution. Figure~\ref{fig:ravan} presents an example actionable recourse scenario, suggesting an increase in savings, credit amount and purpose improves credit risk.

\ignore{We report the local explanations generated by \sys\ in Figure~\ref{fig:exp:related:local:negative:german}. For the negative outcome example, the values of current status, credit history and housing contribute positively toward the outcome. The negative contribution of status and age indicate that a further increase in their values could result in a potential favorable outcome. As is expected in the real world, younger individuals and those with insufficient daily minimum amount in checking accounts are less likely to be considered good credit risks. On the other hand, all current attribute values of the positive outcome example (Figure~\ref{fig:exp:related:local:positive:german}) contribute toward their favorable outcome. Since increasing any of them is unlikely to further improve the outcome, the values do not have a negative contribution.}

\begin{figure}[t]
    \centering
    \subcaptionbox{Effect of status on different age groups. (\texttt{German}) \label{fig:exp:end:contextual:german}}
    {\includegraphics[width=.24\columnwidth]{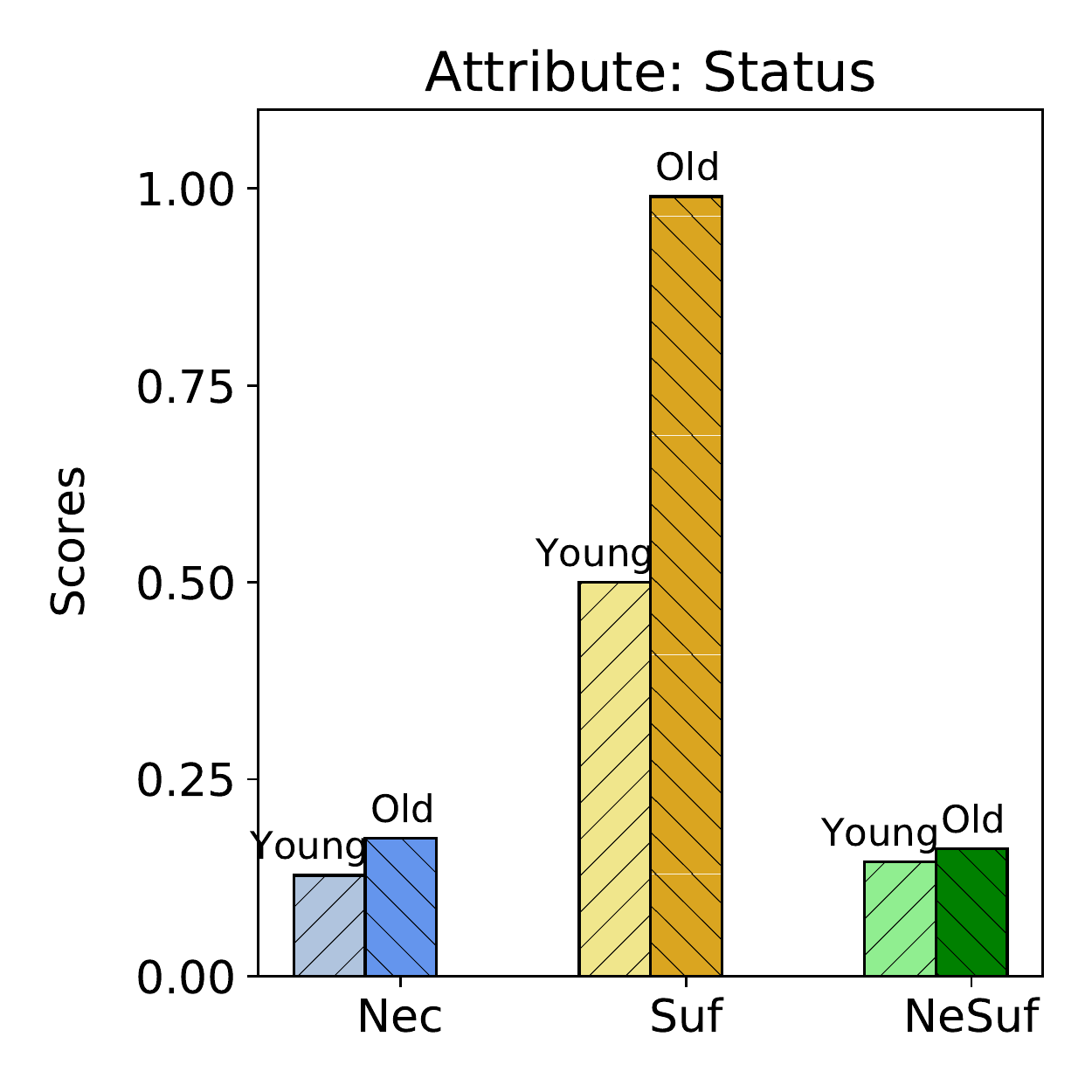}}
    \subcaptionbox{Effect of marital on different age groups. (\texttt{Adult}) \label{fig:exp:end:contextual:adult}}
    {\includegraphics[width=.24\columnwidth]{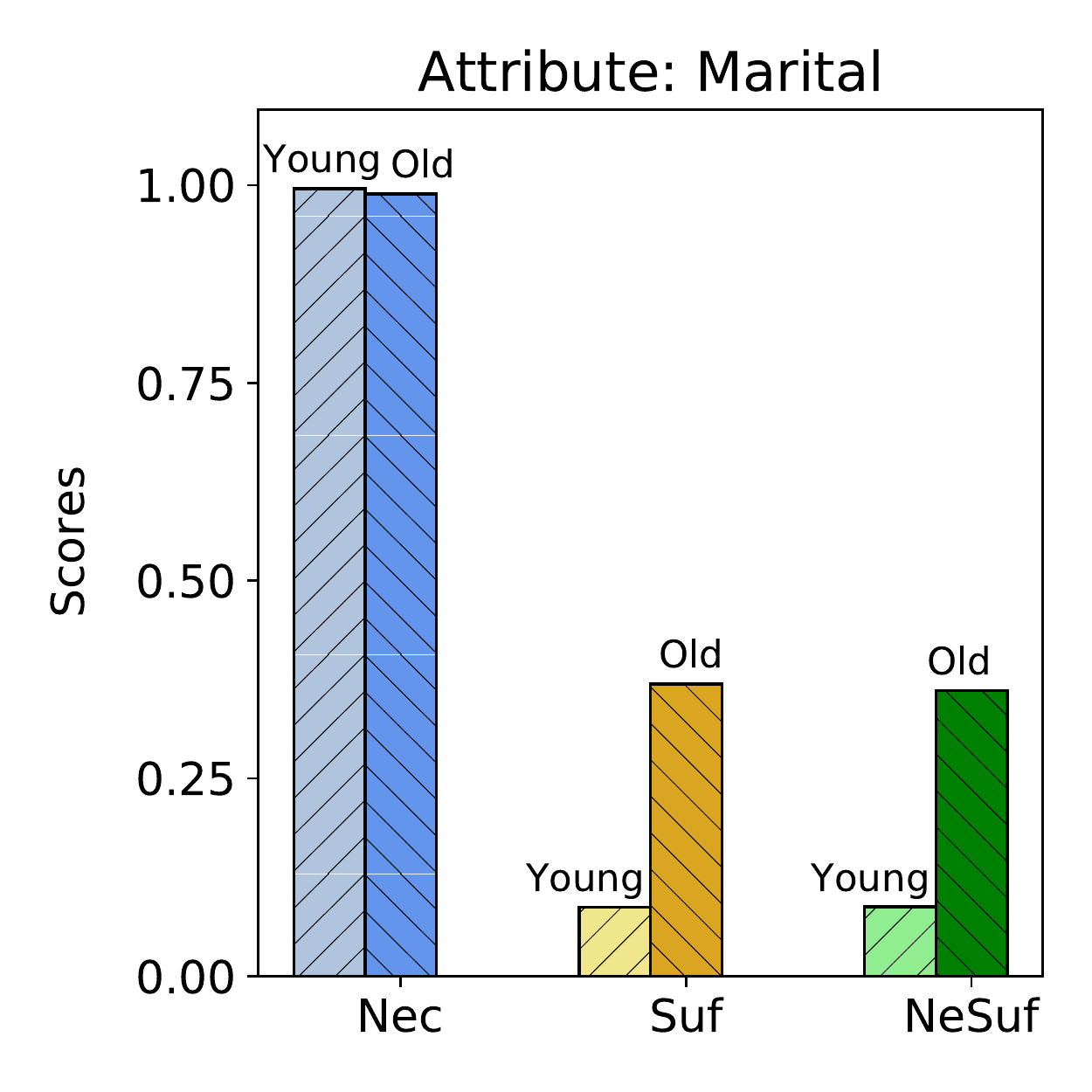}}
     \subcaptionbox{Effect of prior count on race. (\texttt{COMPAS}) \label{fig:exp:end:contextual:compas_prior_count}}
     {\includegraphics[width=.24\columnwidth]{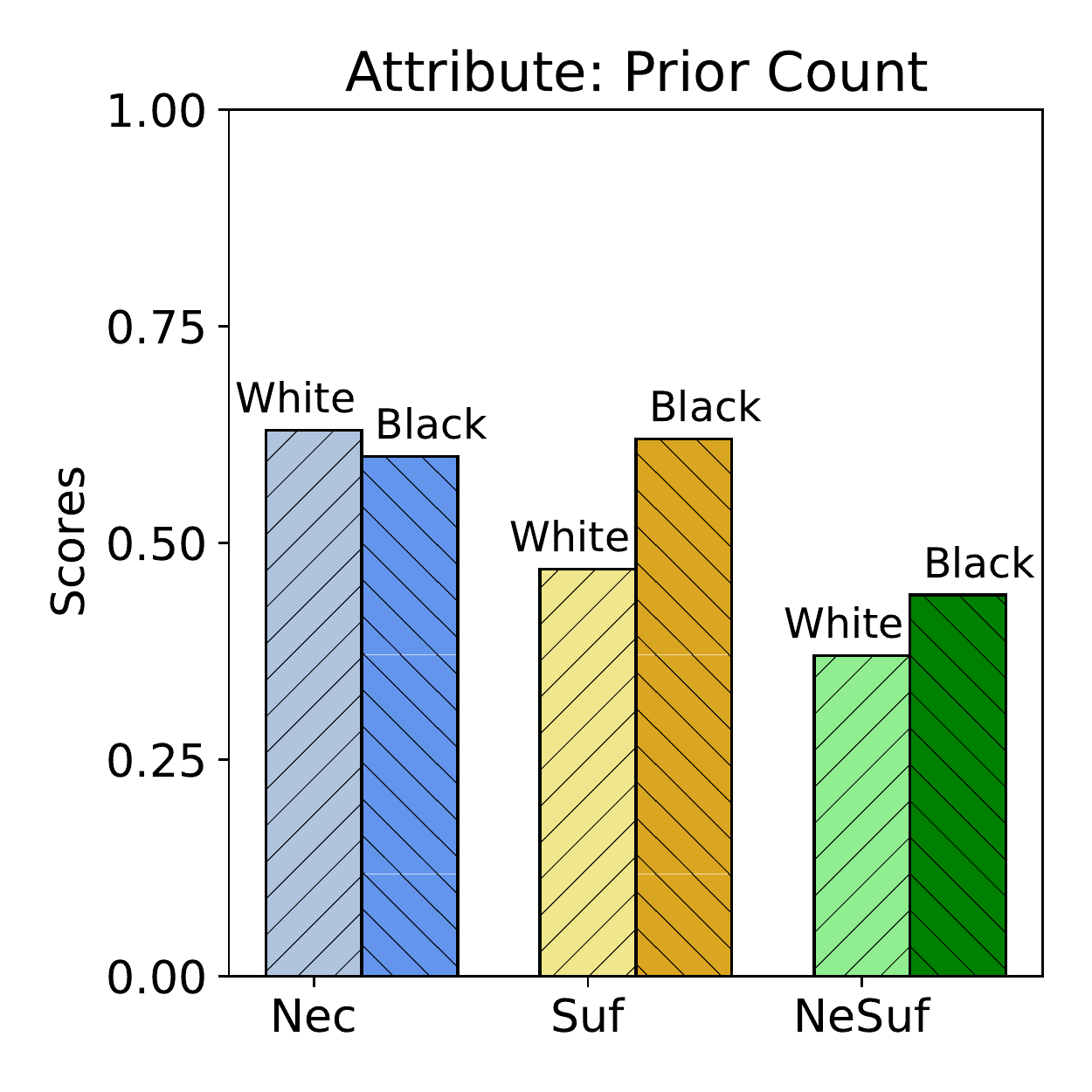}}
   \subcaptionbox{Effect of juvenile crime on race. (\texttt{COMPAS}) \label{fig:exp:end:contextual:compas_juvenile_crime}}
    {\includegraphics[width=.24\columnwidth]{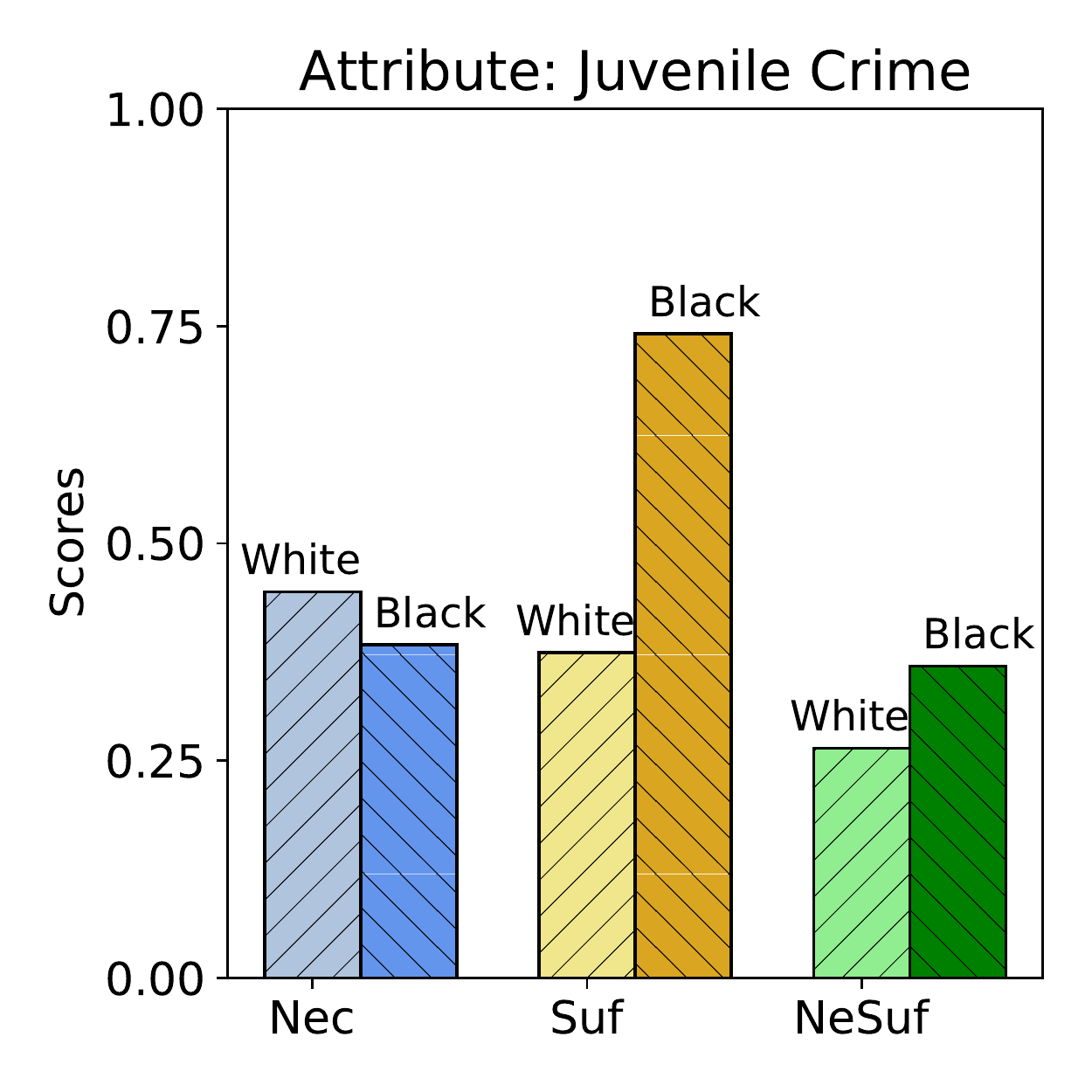}}
      \caption{\sys's contextual explanations show the effect of intervening on an attribute over different sub-populations.  }\label{fig:exp:end:contextual}
\end{figure}

\begin{figure}
\centering
\begin{minipage}{.5\textwidth}
  \centering
  \includegraphics[width=.6\columnwidth]{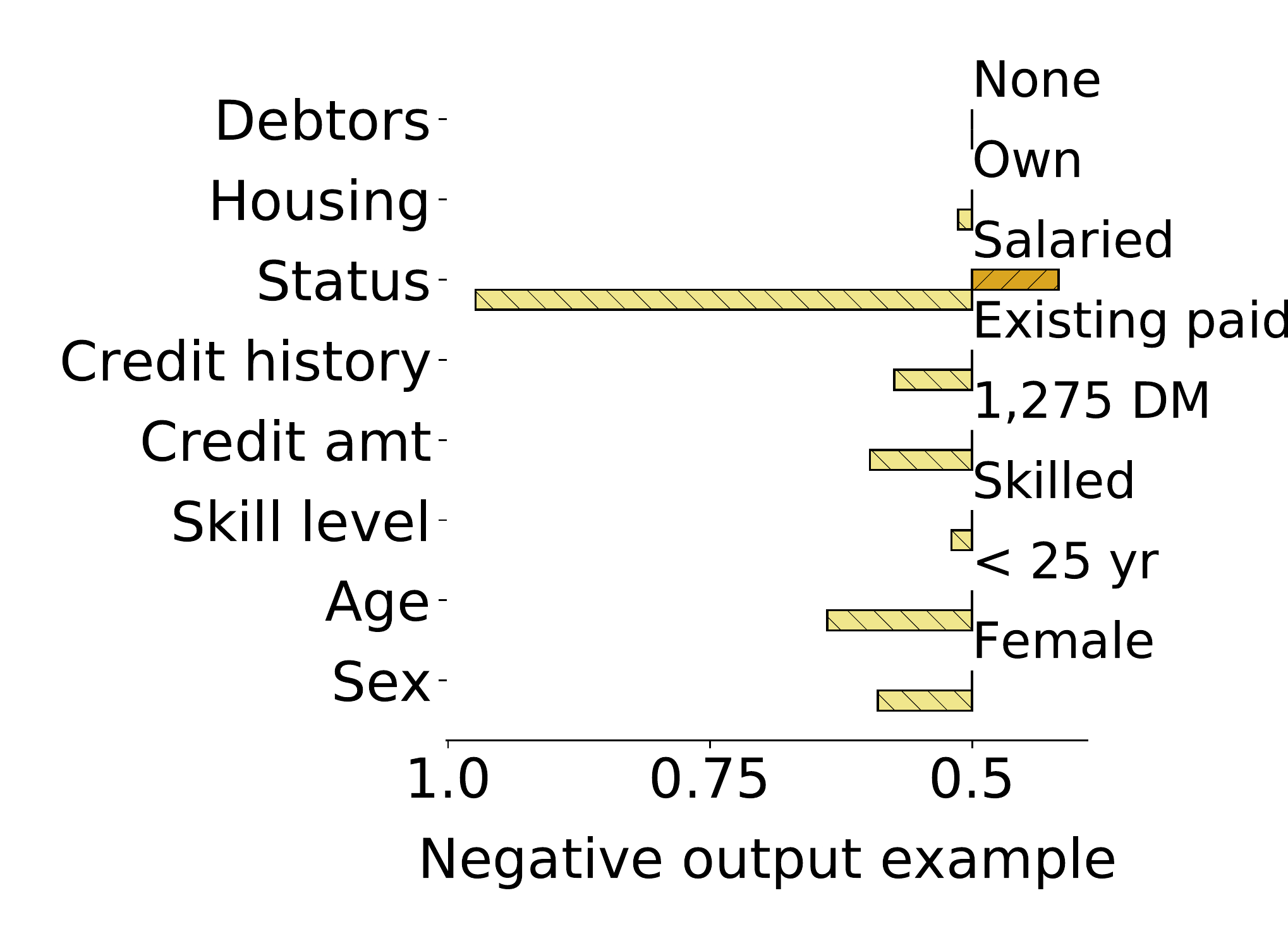}
      \includegraphics[width=.38\columnwidth]{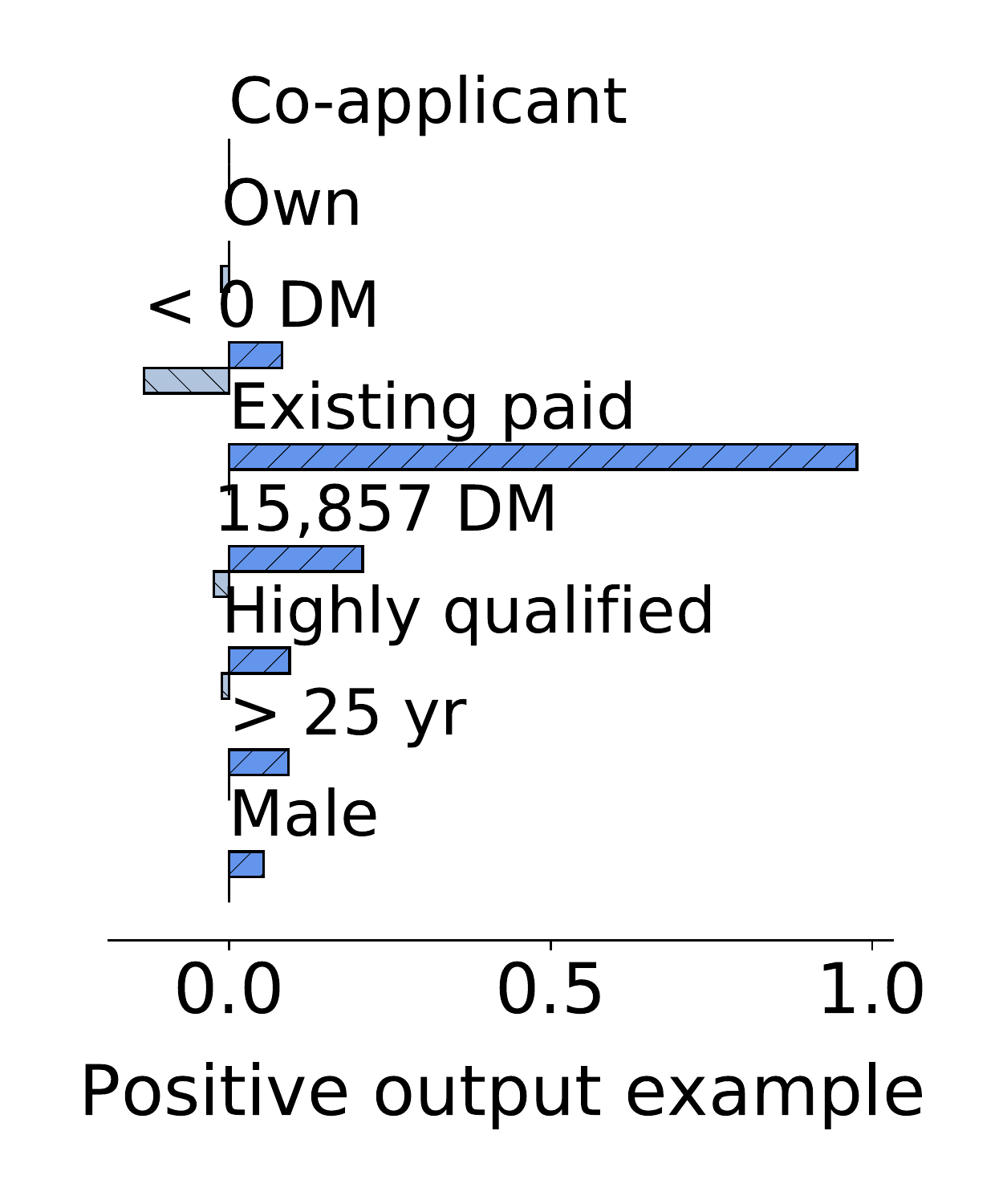}
      \vspace*{-.3cm}		\caption{\sys's local explanations. (\texttt{German})}\label{fig:exp:end:local:german}
\end{minipage}%
\begin{minipage}{.5\textwidth}
  \centering
          \includegraphics[width=.6\columnwidth]{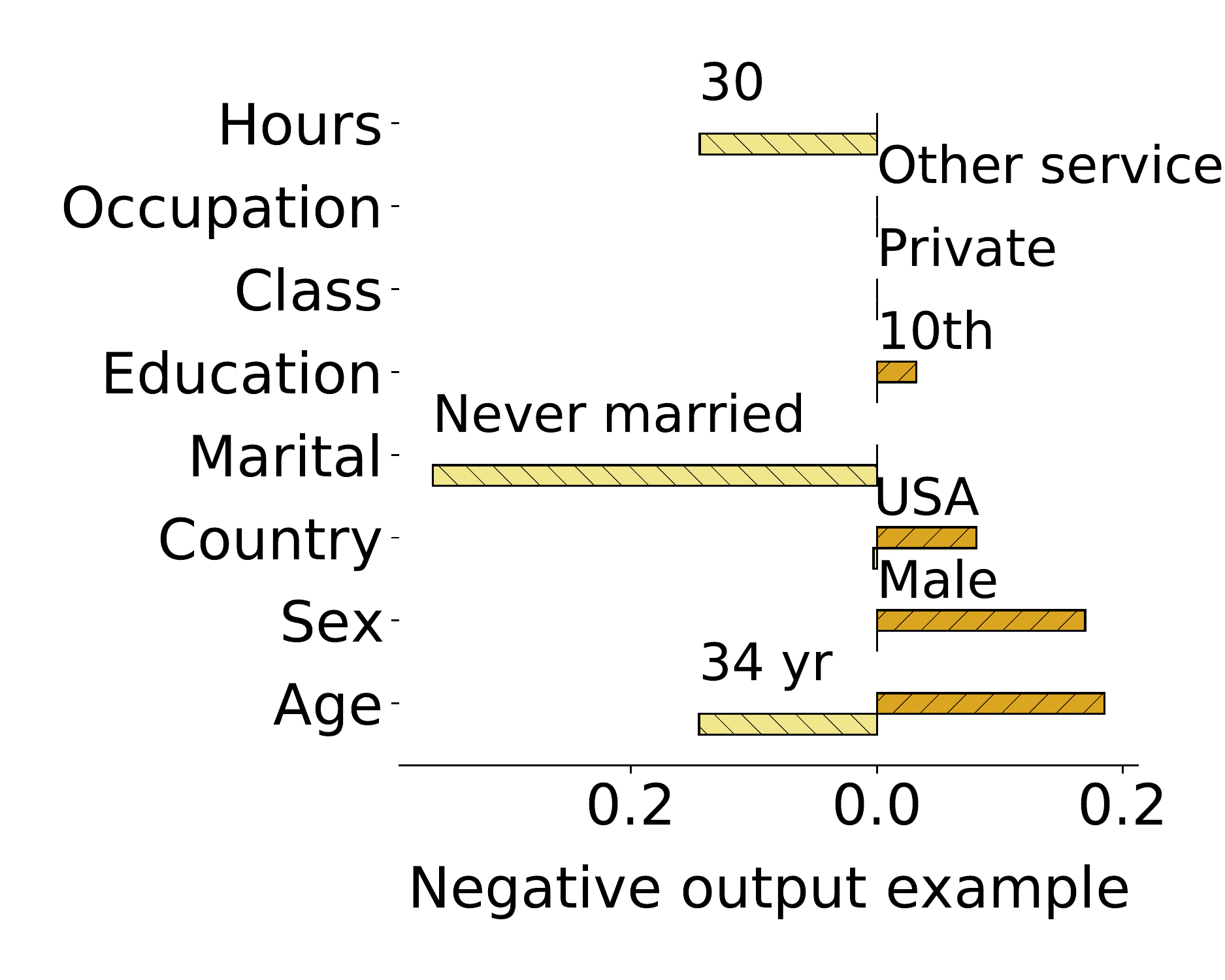}
      \includegraphics[width=.38\columnwidth]{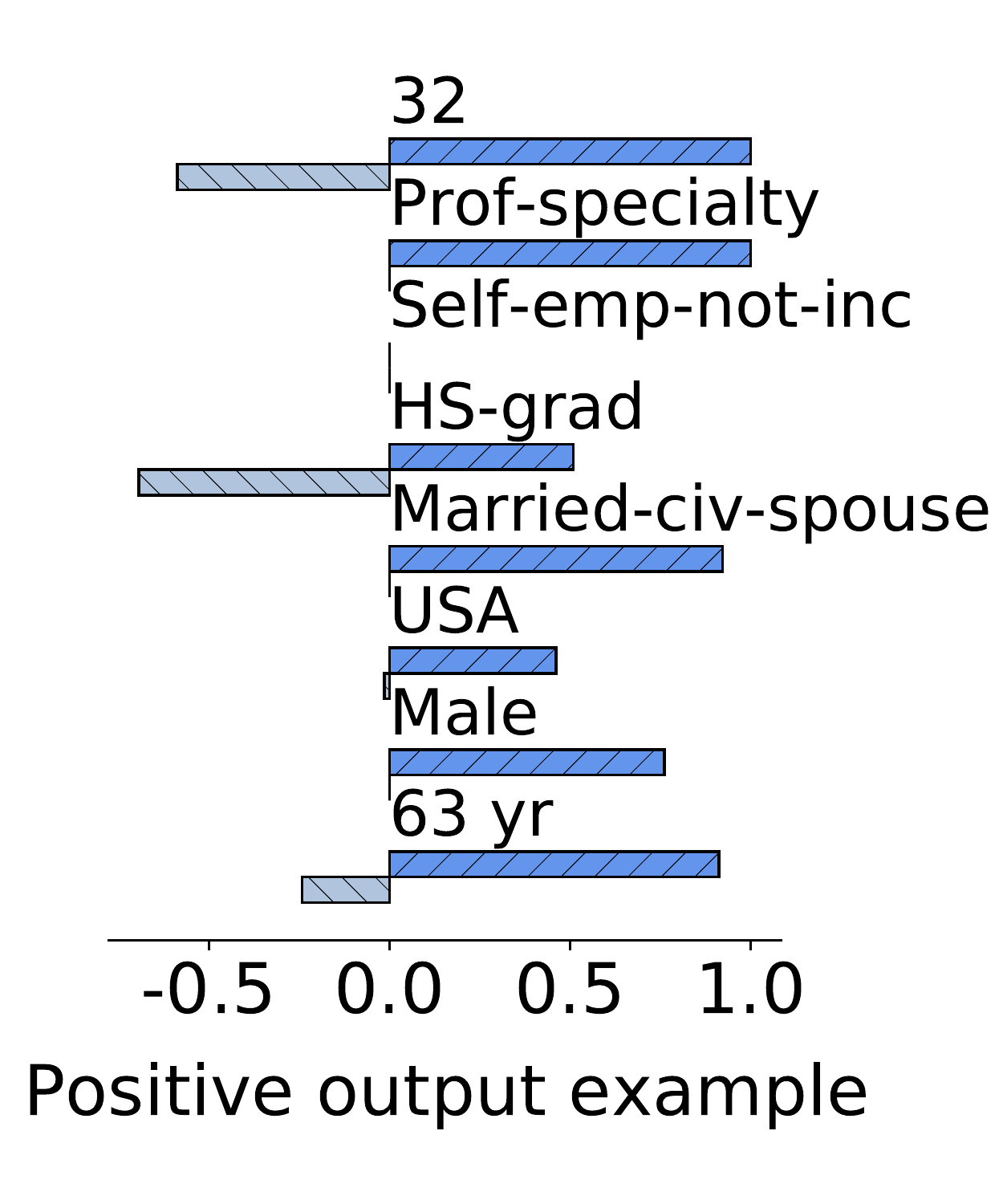}
      \vspace*{-.3cm}		\caption{\sys's local explanations. (\texttt{Adult})}\label{fig:exp:end:local:adult}
\end{minipage}
\end{figure}

\begin{figure}
      \centering
      
      \subcaptionbox{\revb{Negative outcome example.} \label{fig:exp:related:local:negative:drug}}
    {\includegraphics[width=.24\columnwidth]{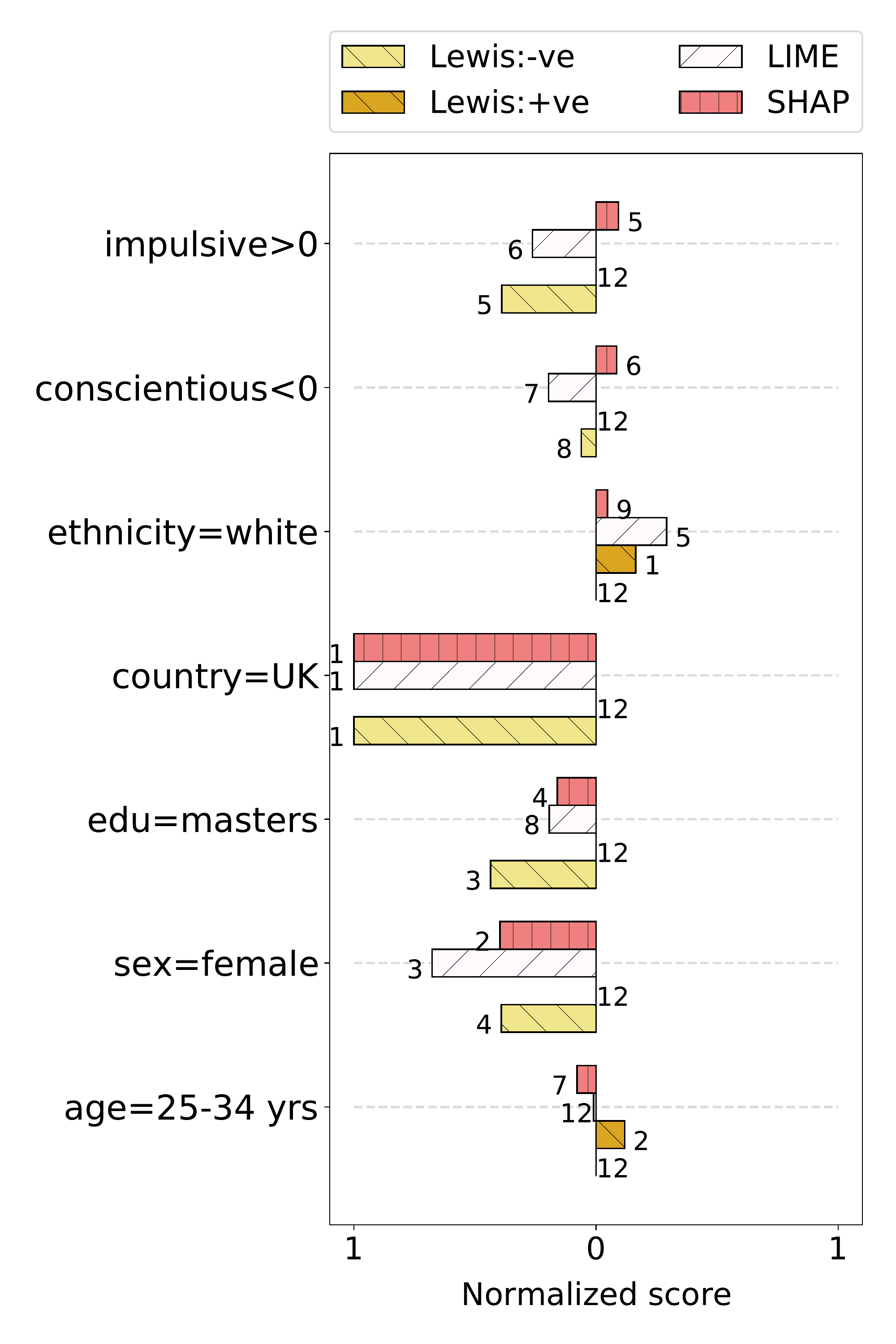}}
        \subcaptionbox{\revb{Positive outcome example.} \label{fig:exp:related:local:positive:drug}}
    {\includegraphics[width=.24\columnwidth]{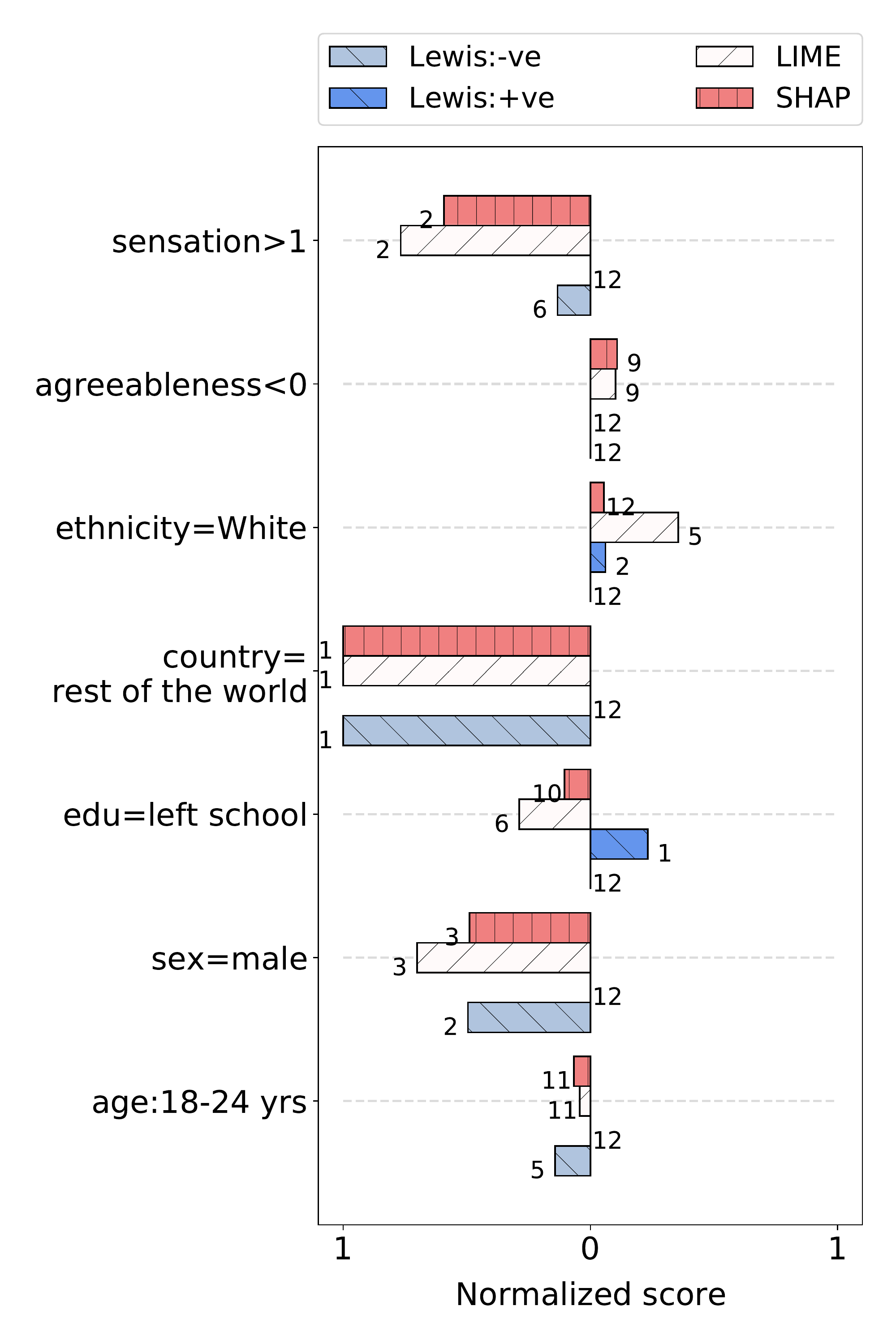}}
      
      \caption{\revb{\sys's local explanations. (\texttt{Drug})}}\label{fig:exp:rel:local:drug}
\end{figure}


\noindent\textbf{Adult.} Several studies~\cite{TAGH+17,10.1109/ICDM.2011.72} have analyzed the impact of gender and age in this dataset. The dataset has been shown to be inconsistent: income attributes for married individuals report household income, and there are more married males in the dataset indicating a favorable bias toward males~\cite{DBLP:conf/sigmod/SalimiGS18}. We, therefore, expect age to be a necessary cause for higher income, but it may not be sufficient since increasing age does not imply that an individual is married. This intuition is substantiated by the high necessity and low sufficiency scores of \texttt{age} in Figure~\ref{fig:exp:end:global:adult}. Furthermore, as shown in Figure~\ref{fig:exp:end:contextual:adult}, changing marital status to a higher value has a greater effect on older than on younger individuals; this effect can be attributed to the fact that compared to early-career individuals, mid-career individuals typically contribute more to joint household income. Consequently, for an individual with a negative outcome  (Figure~\ref{fig:exp:end:local:adult}), marital status and age contribute toward the negative outcome. For an individual with a positive outcome, changing any attribute value is less likely to improve the outcome. However, increasing working hours will further the favorable outcome with a higher probability. We calculated the recourse for the individual with negative outcome and identified that increasing the hours to more than $42$ would result in a high-income prediction.

\noindent\textbf{COMPAS.}  We compare the global explanation scores generated by \sys{} for the COMPAS software used in courts (Figure~\ref{fig:exp:end:global:compasscore}). The highest score of 
\texttt{priors\_ct} validates the insights of previous studies~\cite{compas,salimi2019capuchin} that the number of prior crimes is one of the most important factors determining chances of recidivism. 
Figures
~\ref{fig:exp:end:contextual:compas_prior_count} and
~\ref{fig:exp:end:contextual:compas_juvenile_crime}
present the effect of intervening on  \texttt{prior} crime count 
and \texttt{juvenile} crime count, respectively, 
on the software score (note that for these explanations, we use the prediction scores from the COMPAS software, not the classifier output). We observe that both the attributes have a higher sufficiency for  \texttt{Black}  compared to \texttt{White}, indicating that an increase in prior crimes and juvenile crimes is more detrimental for Blacks compared to  Whites. A reduction in these crimes, on the other hand, benefits Whites more than Blacks, thereby validating the inherent bias in COMPAS scores. We did not perform recourse analysis as the 
attributes describe past crimes and, therefore, are not actionable.

\revb{\noindent\textbf{Drug.} This dataset has been studied to understand the variation in drug patterns across demographics and the effectiveness of various sensation measurement features toward predicting drug usage. Figure~\ref{fig:exp:end:global:multi-class} compares the global scores with respect to the outcome that the drug was used atleast once in lifetime. Previous studies~\cite{fehrman2017factor} have found that consumption of the particular drug is common in certain countries, as substantiated by the high necessity and sufficiency scores of country. Furthermore, intuitively, individuals with a higher level of education are more likely to be aware of the effects of drug abuse and hence, less likely to indulge in its consumption. This intuition is supported by the observation in Figure~\ref{fig:exp:related:local:negative:drug}: a higher education level contributes towards the negative drug consumption outcome, and in Figure~\ref{fig:exp:related:local:positive:drug}: a lower education level  contributes positively toward the drug consumption outcome.
We observe similar conclusions for the explanations with respect to a different outcome such as drug used in the last decade.
}

\begin{figure}
      \centering
         \subcaptionbox{\revc{\texttt{Adult} + XGBoost} \label{fig:exp:end:relative:adultxgboost}}
        {\includegraphics[width=.24\columnwidth]{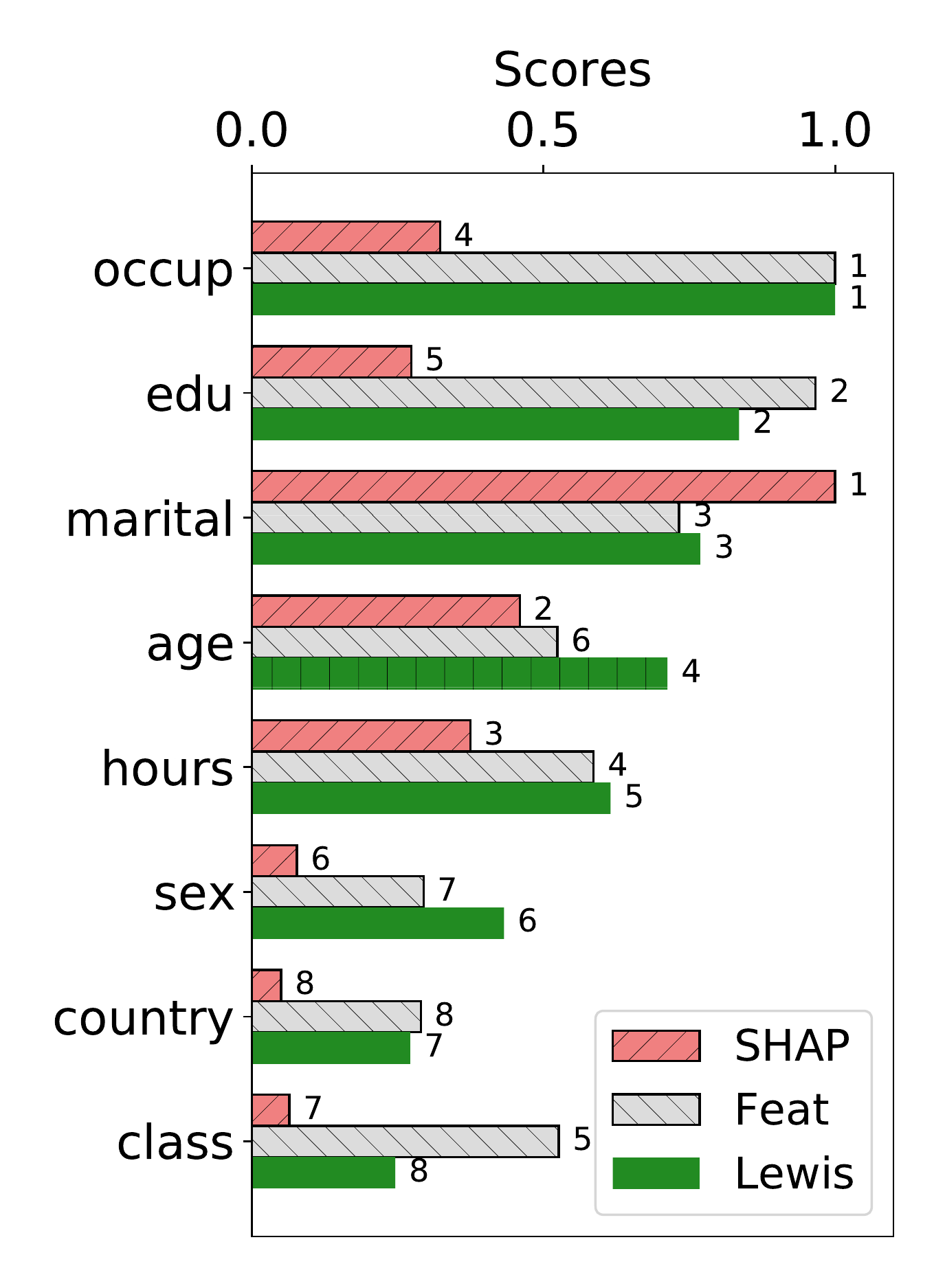}}
        \subcaptionbox{\revc{\texttt{Adult} + Neural Networks} \label{fig:exp:end:relative:adultneural}}
        {\includegraphics[width=.24\columnwidth]{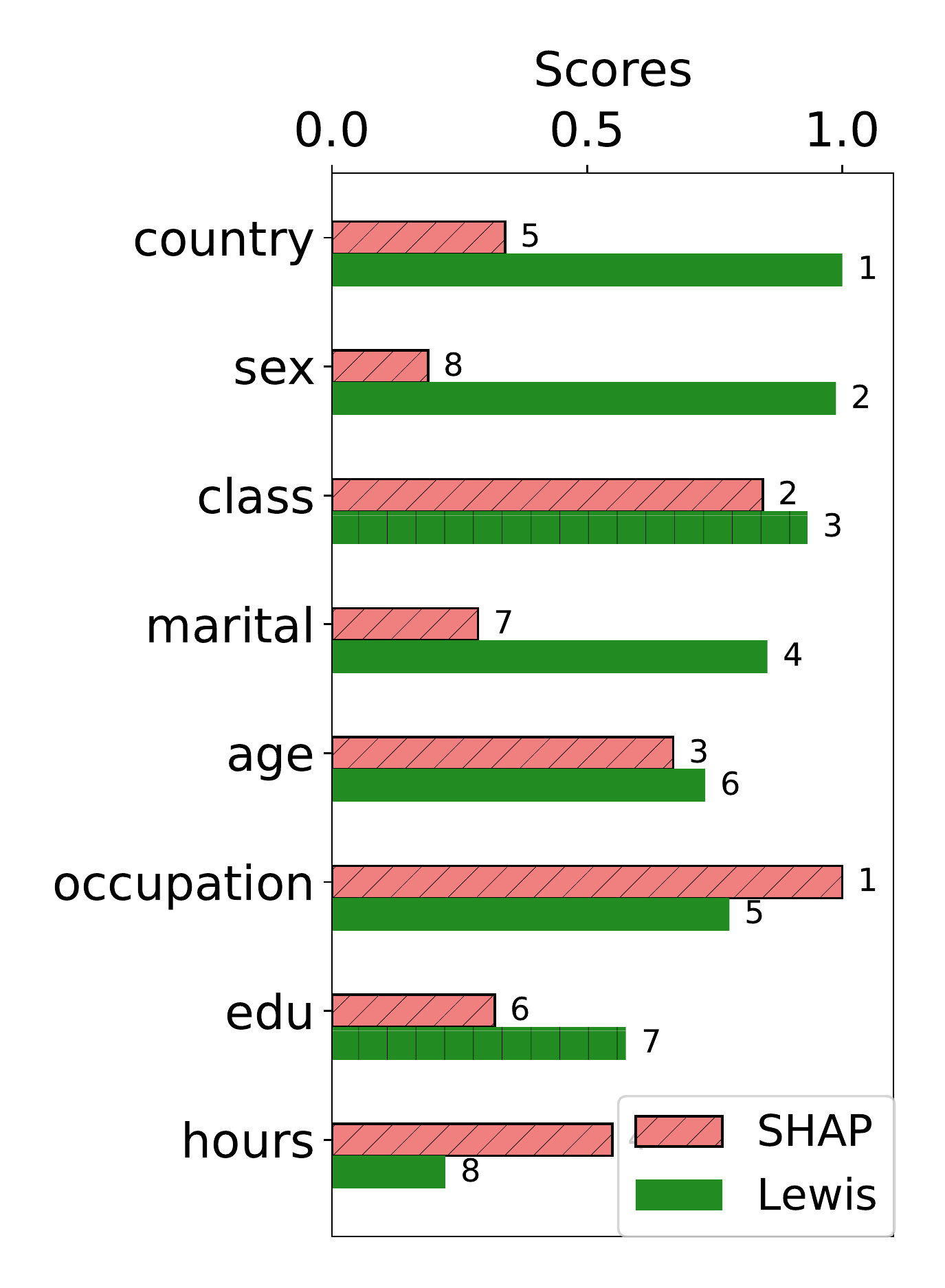}}
      \caption{\revc{Generalizability of \sys{} to black-box algorithms.}}\label{fig:exp:end:relative:adultnew}
\end{figure}
\noindent \revc{\textbf{Generalizability of \sys{} to black-box algorithms.} In Figure~\ref{fig:exp:end:relative:adultnew}, we present the global explanations generated by \sys{} for black-box algorithms that are harder to interpret and are likely to violate the monotonicity assumption, such as XGBoost and feed forward neural networks, and report the necessity and sufficiency score for each classifier. For ease in deploying neural networks, we conducted this set of experiments on \texttt{Adult} which is our largest dataset. We observed that different classifiers rank attributes differently depending upon the attributes they deem important. For example, the neural network learns class as the most important attribute. Since country and sex have a causal effect on class, \sys{} ranks these three attributes higher than others (see Section~\ref{exp:related} for a detailed interpretation of the results).} 

\reva{
\noindent\textbf{Key takeaways.} 
(1)~The explanations generated by \sys\ capture causal dependencies between attributes, and are applicable to any black-box algorithm. (2)~\sys\ has proved effective in determining attributes causally responsible for a favorable outcome. (3)~Its contextual explanations, that show the effect of particular interventions on sub-populations, are aligned with previous studies. (4)~The local explanations offer fine-grained insights into the contribution of attribute values toward the outcome of an individual. (5)~For individuals with an unfavorable outcome, whenever applicable, \sys\ provides recourse in the form of actionable interventions.

}

\subsection{Comparing \sys\ to Other Approaches}
\label{exp:related}

\begin{figure*}
      \centering
      \subcaptionbox{\texttt{German} \label{fig:exp:related:global:german}}
        {\includegraphics[width=.24\columnwidth]{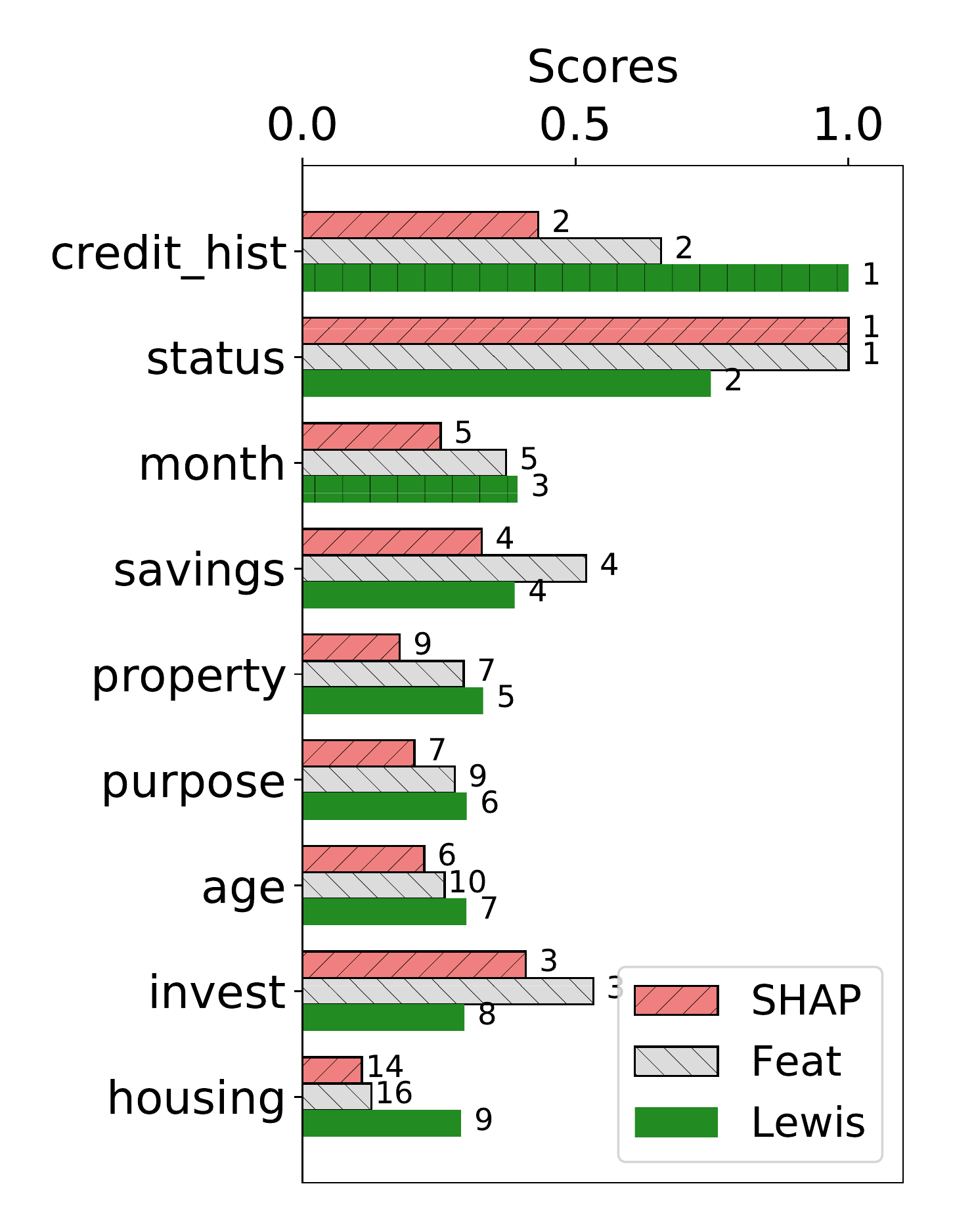}}
      \subcaptionbox{\texttt{Adult} \label{fig:exp:related:global:adult}}
        {\includegraphics[width=.24\columnwidth]{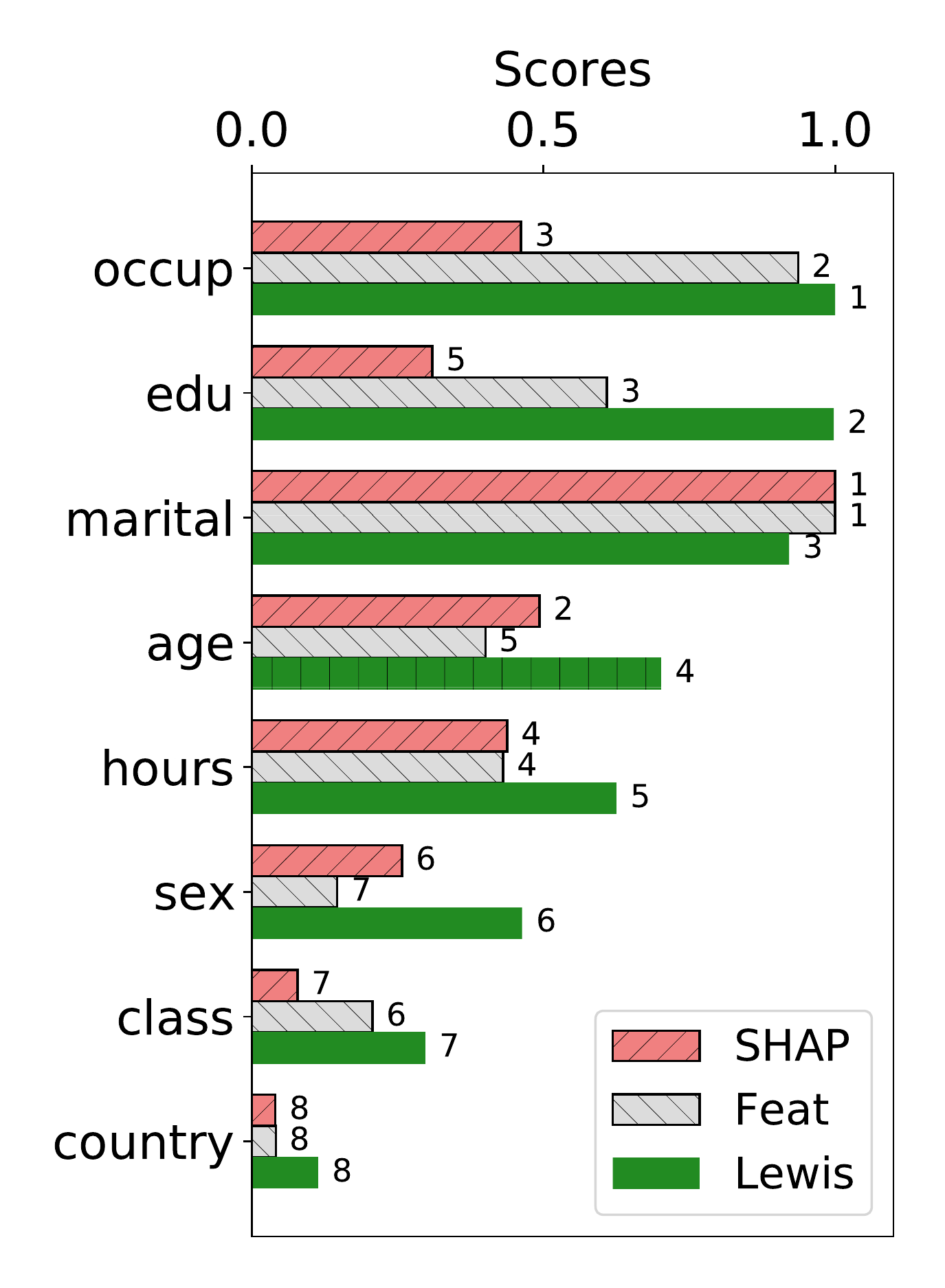}}
         \subcaptionbox{\texttt{Compas} -- Software score \label{fig:exp:related:global:compasscore}}
        {\includegraphics[width=.24\columnwidth]{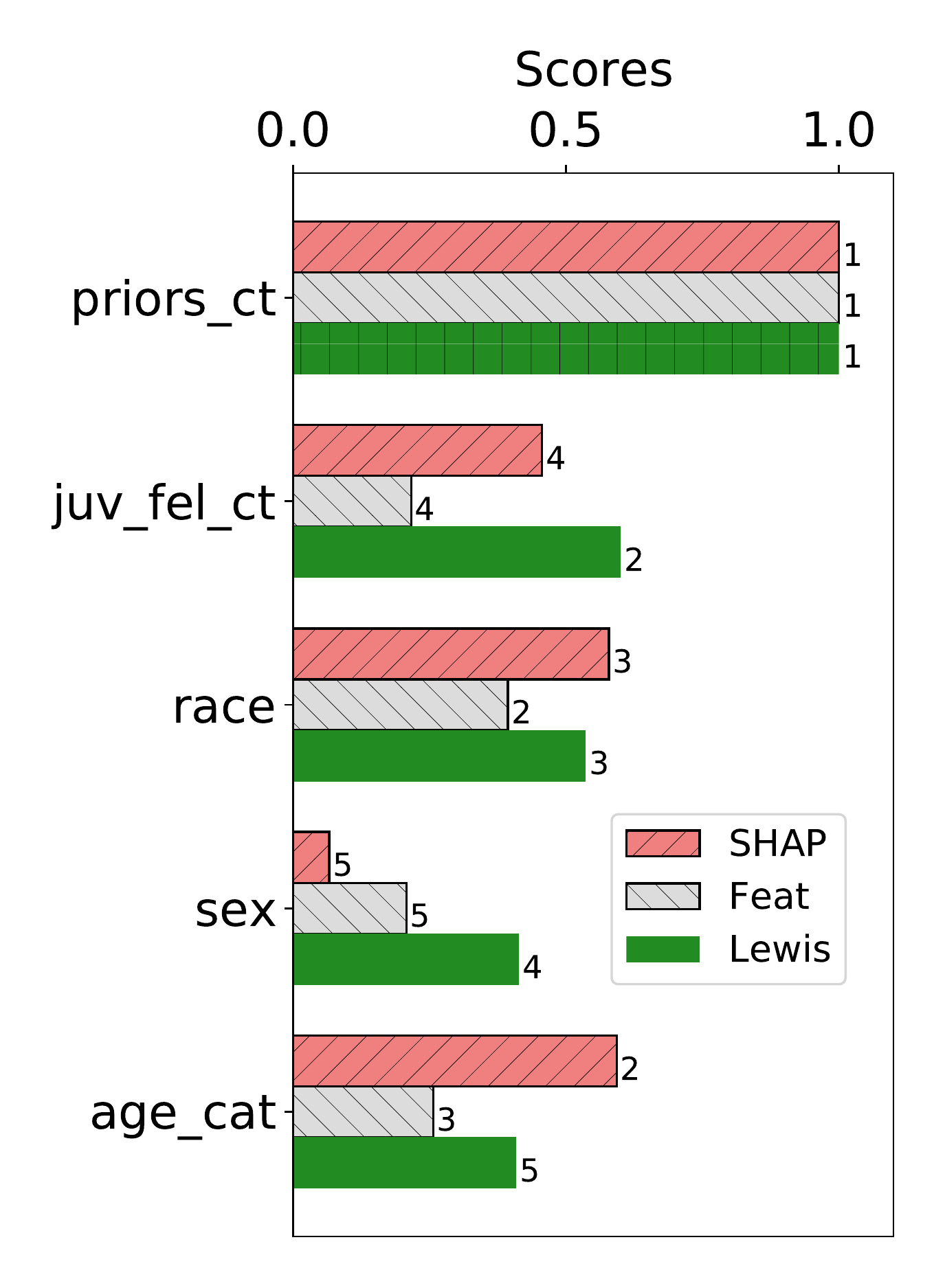}}
        \subcaptionbox{\revb{\texttt{Drug}} \label{fig:exp:related:global:multi-class}}
        {\includegraphics[width=.24\columnwidth]{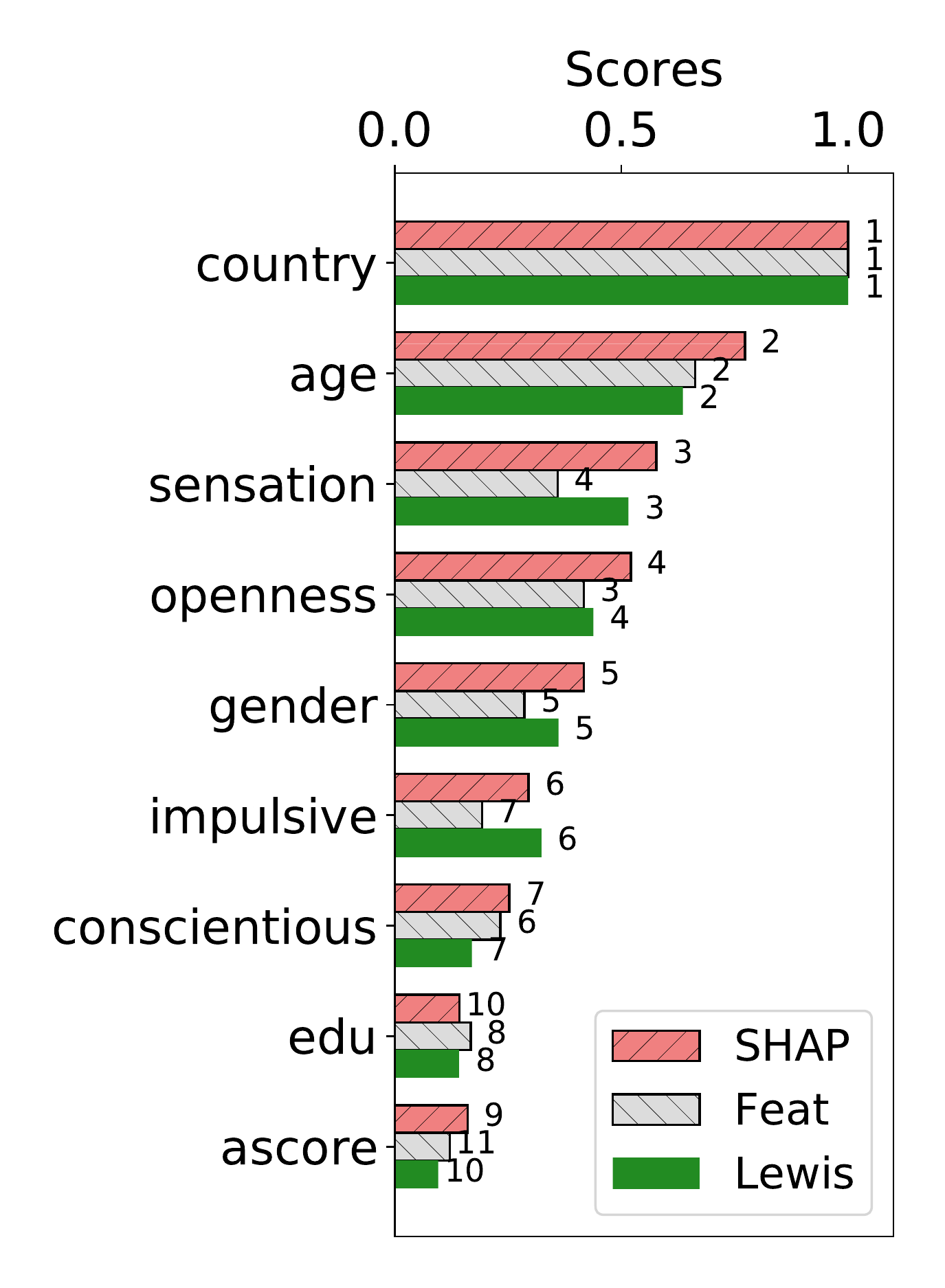}}
        \vspace{-2mm}
      \caption{Comparing different global explanation methods: \texttt{SHAP} and \texttt{Feat} fail to account for causal relationships in the data that are effectively captured by \sys.}\label{fig:exp:end:global}
      \label{fig:exp:related:global}
\end{figure*}

We compared the global and local explanations generated by \sys{} to existing approaches used for interpreting ML algorithms: \texttt{SHAP}~\cite{DBLP:conf/nips/LundbergL17}, \texttt{LIME}~\cite{ribeiro2016should} and feature importance (\texttt{Feat})~\cite{10.1023/A:1010933404324}.
\texttt{SHAP}
explains the difference between a prediction and the {\em global} average prediction, \texttt{LIME}
explains the difference from a {\em local} average prediction, and \texttt{Feat} measures the increase in an algorithm's prediction error after permutating an attribute's values. \texttt{LIME} provides local explanations, \texttt{Feat} provides global explanations and \texttt{SHAP} generates both global and local explanations. \texttt{LIME} and \texttt{SHAP} provide marginal contribution of an attribute to classifier prediction and are not directly comparable to \sys's probabilistic scores. However, since all the methods measure the importance of attributes in classifier predictions, we report the relative ranking of attributes generated on their normalized scores, and present a subset of attributes ranked high by any of the methods. We report the maximum $\nsuf_x$ score of an attribute obtained by \sys\ on all of its value pairs. We also compared the recourse generated by \sys{} with \texttt{LinearIP}.  (We contacted the authors of \cite{karimi2020algorithmic} but do not use it in evaluation since their technique does not work for categorical actionable variables). We used open-source implementations of the respective techniques.

\noindent\textbf{German.} In Figure~\ref{fig:exp:related:global:german}, 
note that \texttt{housing} is ranked higher by \sys\ than by \texttt{Feat} and \texttt{SHAP}. The difference lies in the data: \texttt{housing=own} is highly correlated with a positive outcome. However, due to a skewed distribution (there are $\sim10\%$ of instances where \texttt{housing=own}), random permutations of \texttt{housing} do not generate new instances, and \texttt{Feat} is unable to identify it as an important attribute. \sys\ uses the underlying causal graph to capture the causal relationship between the two attributes.

In Figures~\ref{fig:exp:related:local:negative:german} and \ref{fig:exp:related:local:positive:german}, we report the rankings obtained by \texttt{LIME}, \texttt{SHAP} and \sys\ on two instances that respectively have negative and positive predicted outcomes. Age and account status have a high negative contribution toward the outcome in Figure~\ref{fig:exp:related:local:negative:german}, indicating that an increase in either is likely to reverse the decision. Intuitively, with age, continued employment and improved account status, individuals tend to have better savings, credit history, housing, etc., which, in turn, contribute toward a positive outcome. \sys's ranking captures this causal dependency between the attributes, which is recorded by neither \texttt{SHAP} nor \texttt{LIME}.

To compare recourse generated by \sys\ and \texttt{LinearIP}, we tested them on the example for \negativeuser\ in Figure~\ref{fig:ravan}. While both the methods identify the same solution for small thresholds, \texttt{LinearIP} did not return any solution for success threshold $> 0.8$. In contrast to \sys{} that generalizes to black-box algorithms,  \texttt{LinearIP} depends on linear classifiers and offers recommendations that do not account for the causal relationship between attributes.\\

\begin{figure*}
      \centering
      \subcaptionbox{Negative outcome (\texttt{German}) \label{fig:exp:related:local:negative:german}}
       {\hspace{-7pt}\includegraphics[width=.24\columnwidth]{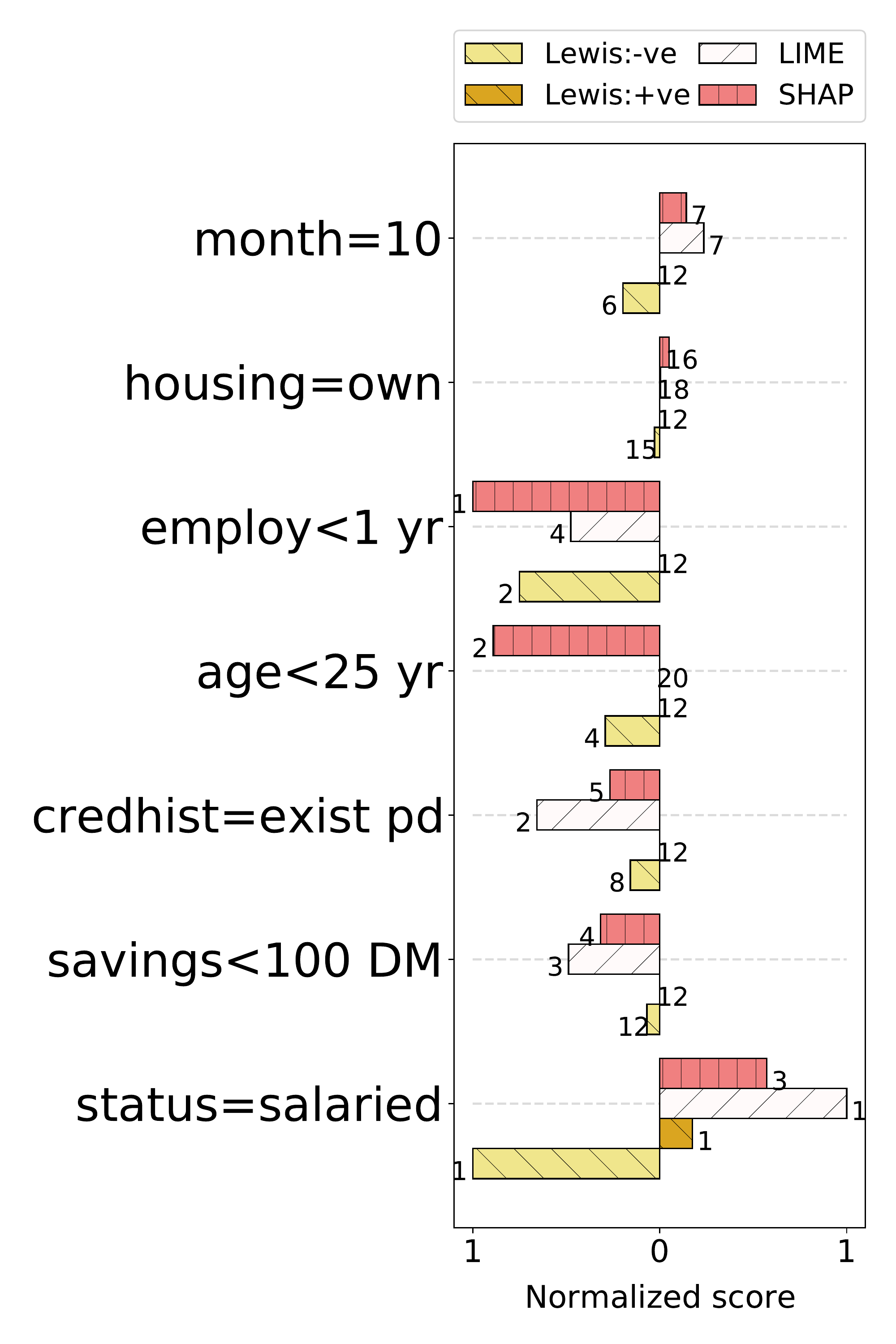}}
      \subcaptionbox{Positive outcome (\texttt{German}) \label{fig:exp:related:local:positive:german}}
       {\hspace{-7pt}\includegraphics[width=.24\columnwidth]{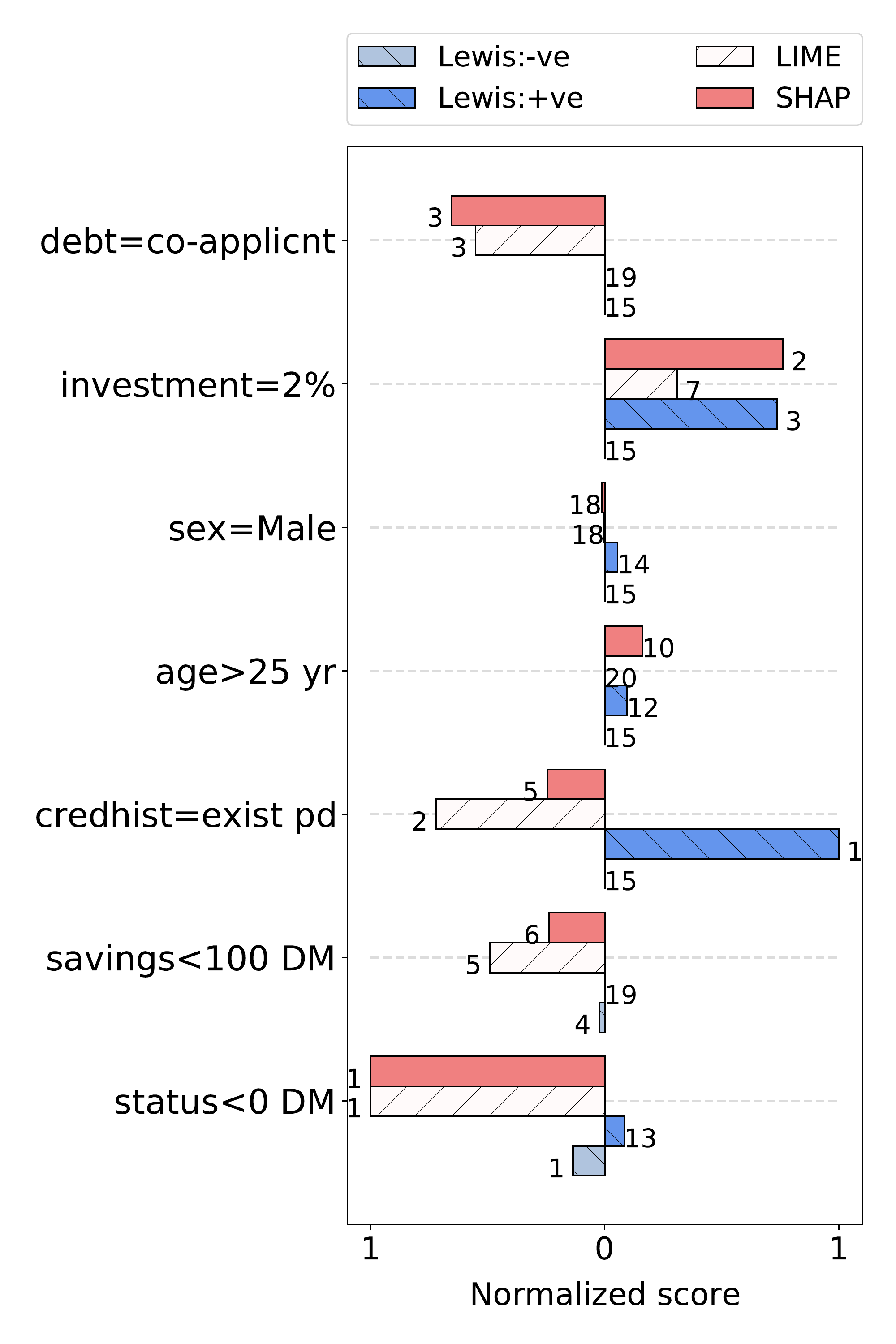}}
      \subcaptionbox{Negative outcome (\texttt{Adult}) \label{fig:exp:related:local:negative:adult}}
      {\hspace{-7pt}\includegraphics[width=.24\columnwidth]{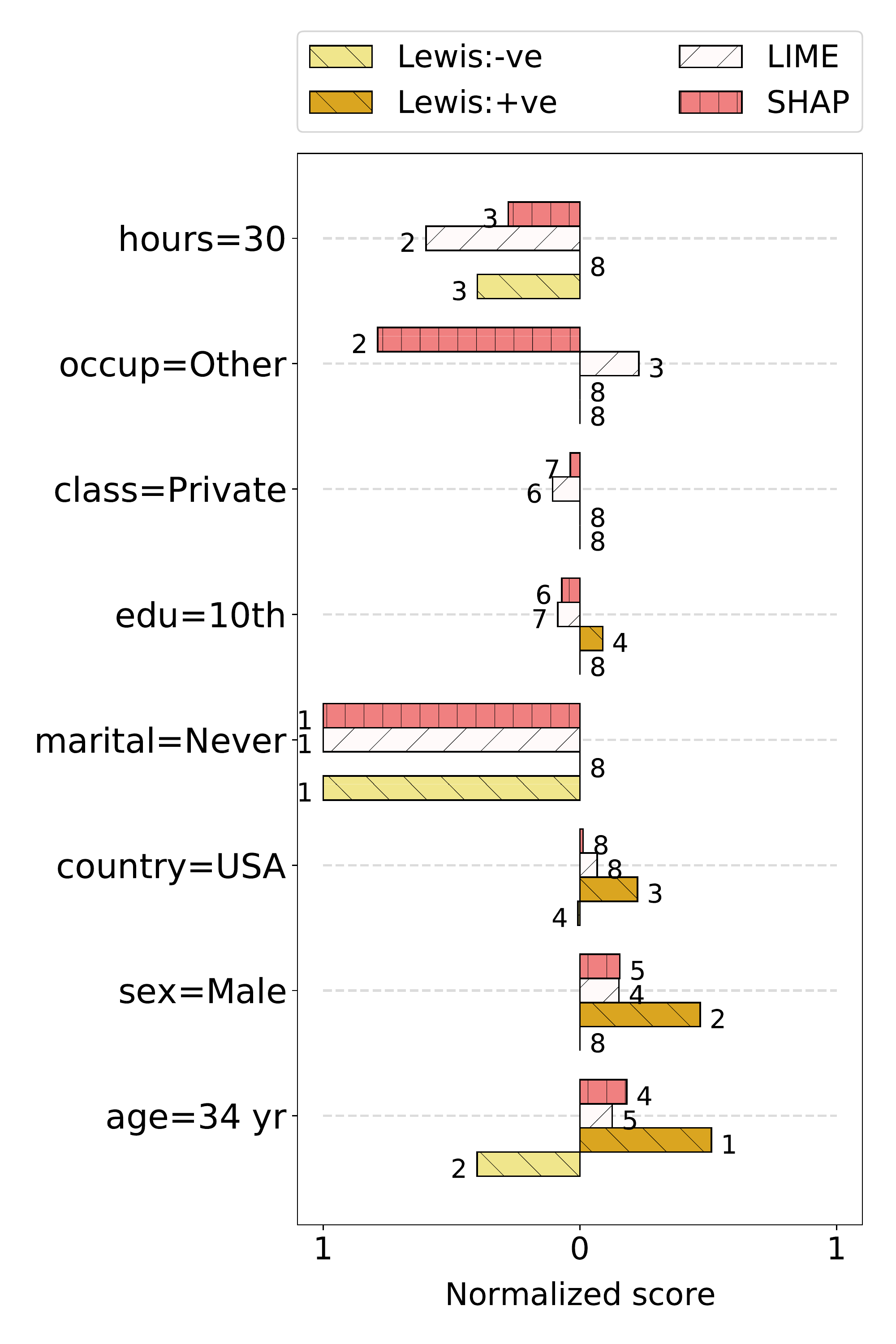}} 
      \subcaptionbox{Positive outcome (\texttt{Adult}) \label{fig:exp:related:local:positive:adult}}
      {\hspace{-7pt}\includegraphics[width=.24\columnwidth]{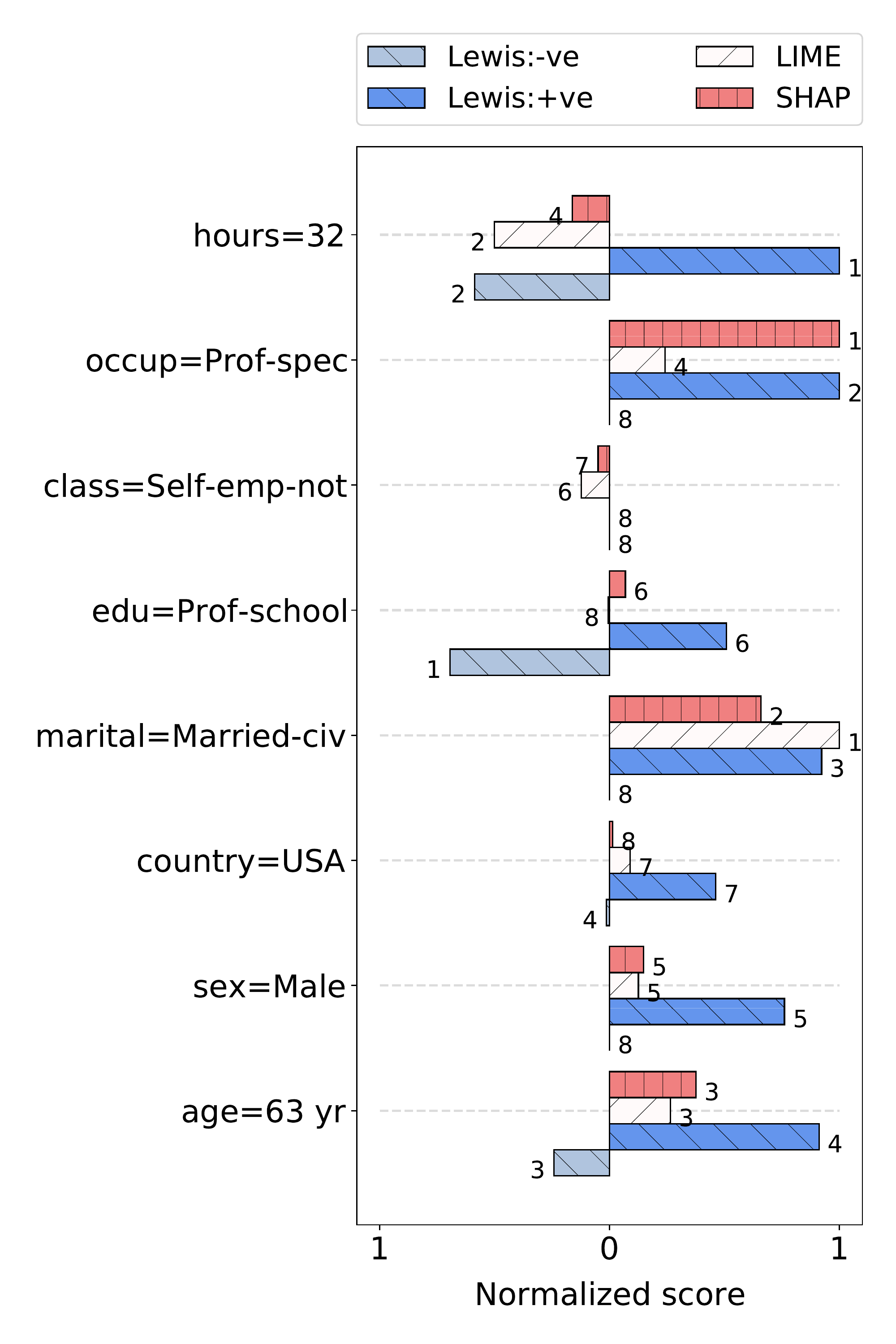}}
      \vspace{-1mm}
      \caption{Comparing different local explanation methods. \texttt{SHAP} and \texttt{LIME} explain the output of an instance in terms of its difference from the global or local average prediction; \sys{} explains it in terms of the underlying causal graph.}
\end{figure*}
\vspace{-.3cm}
\noindent \textbf{Adult.}
In Figure~\ref{fig:exp:related:global:adult},
the ranking of attributes generated by \sys{} and \texttt{Feat} matches observations in prior literature that consider occupation, education and marital status to be the most important attributes. However, \texttt{SHAP} picks on the \textit{correlation} of age with marital status and occupation (older individuals are more likely to be married and have better jobs), and ranks it higher.
\revc{
The rankings are similar for XGBoost (Figure~\ref{fig:exp:end:relative:adultxgboost}) and Random forest (Figure~\ref{fig:exp:related:global:adult}) but different for the neural network (Figure~\ref{fig:exp:end:relative:adultneural}).
We investigated the outputs and observed that the prediction of neural networks differs from that of random forest and XGBoost for more than $20\%$ of the test samples, leading to varied ranking of attributes. Additionally, the class of an individual is ranked important by the classifier. Since country and sex have a causal impact on class, it justifies their high ranks as generated by \sys{}. In Figure~\ref{fig:exp:end:relative:adultneural}, we do not report the scores for \texttt{Feat} as it does not support neural networks.} 
  
In Figures~\ref{fig:exp:related:local:negative:adult} and~\ref{fig:exp:related:local:positive:adult}, we compare \sys{} with local explanation methods \texttt{LIME} and \texttt{SHAP}. Consistent with existing studies on this dataset, \sys{} recognizes the negative contribution of unmarried marital status and positive contribution of sex=male toward the negative outcome.
For the positive outcome example, \sys{} identifies that age, sex and country have a high positive contribution toward the outcome due to their causal impact on attributes such as occupation and marital status (ranked higher by \texttt{SHAP} and \texttt{LIME}). We also observed that  the results of \texttt{SHAP} are not stable across different iterations.


\noindent \textbf{COMPAS.} Since COMPAS scores were calculated based on criminal history and indications of juvenile delinquency~\cite{compas}, the higher ranking of juvenile crime history by \sys{} is justified in Figure~\ref{fig:exp:related:global:compasscore}. Note that bias penetrated into the system due to the correlation between demographic and  non-demographic attributes. \texttt{SHAP} and \texttt{Feat} capture this correlation and rank age higher than juvenile crime history.

\revb{\noindent \textbf{Drug.}  Figure~\ref{fig:exp:related:global:multi-class} shows that all techniques have a similar ordering of attributes with country and age being most crucial for the desired outcome.
Comparing the local explanations of \sys{} with \texttt{SHAP} and \texttt{LIME} (Figure~\ref{fig:exp:rel:local:drug}), we observe that \sys{} correctly identifies the negative contribution of higher education toward negative drug consumption prediction and the positive contribution of a lower level of education toward a positive drug consumption prediction.}

\vspace{-.3cm}
\subsection{Correctness of \sys's explanations} 

Since ground truth is not available in real-world data, we evaluate the correctness of \sys{} on the \texttt{German-Syn} dataset.

{
\noindent\textbf{Correctness of estimated scores} In Figure~\ref{fig:exp:synthetic:germansn}, we compare the global explanation scores of different variables  with ground truth necessity and sufficiency score estimated using Pearl’s three-step procedure discussed in equation~(\ref{eq:abduction}) (Section~\ref{sec:back}). We present the comparison for a non-linear  regression based black-box algorithm with respect to outcome $o=0.5$.
The average global explanation scores returned by \sys{} are consistently similar to ground truth estimates, thereby validating the correctness of Proposition~\ref{prop:iden:mon:ci}. \texttt{SHAP} and \texttt{Feat} capture the correlation between the input and output attributes, and rank \texttt{Status} higher than \texttt{Age} and \texttt{Sex} which are assigned scores close to $0$. These attributes do not directly impact the output but indirectly impact it through \texttt{Status} and \texttt{Saving}. This experiment validates the ability of \sys{} in capturing causal effects between different attributes, estimate explanation scores accurately and  present actionable insights as compared to \texttt{SHAP} and \texttt{Feat}.
} To understand the effect of the number of samples  {on the scores estimated by \sys{},} 
we compare the $\nsuf$ scores of \texttt{status}  for different sample sizes in Figure~\ref{fig:exp:synthetic:sample}. 
We observe that the variance in estimation is reduced with an increase in sample size and scores converge to ground truth estimates for larger samples.

\revc{
\noindent\textbf{Robustness to violation of Monotonicity}
To evaluate the impact of non-monotonicity on the explanation scores generated by \sys{}, we changed the structural equations for the causal graph of \texttt{German-Syn} to simulate non-monotonic effect of Age on the prediction attribute.  This data was used to train random forest and XGBoost classifiers. We measured \textit{monotonicity violation} as $\Lambda_{\texttt{viol}}=\pr[o_{X\leftarrow x}'| o, x']$. Note that $\Lambda_{\texttt{viol}}=0$ implies monotonicity and higher $\Lambda_{\texttt{viol}}$ denotes higher violation of monotonicity.  We observed that the scores estimated by \sys{} differ from ground truth estimates by less than $5$\%, as long as the monotonicity violation is less than $0.25$. Furthermore, the relative ranking of the attributes remains consistent with the ground truth ranking calculated using equation~(\ref{eq:abduction}). 
This experiment demonstrates that the explanations generated by \sys{} are robust to slight violation in monotonicity.
}

\begin{figure}
      \centering
        \subcaptionbox{{Quality of the estimates.} \label{fig:exp:synthetic:germansn}}
        {\includegraphics[width=.32\columnwidth]{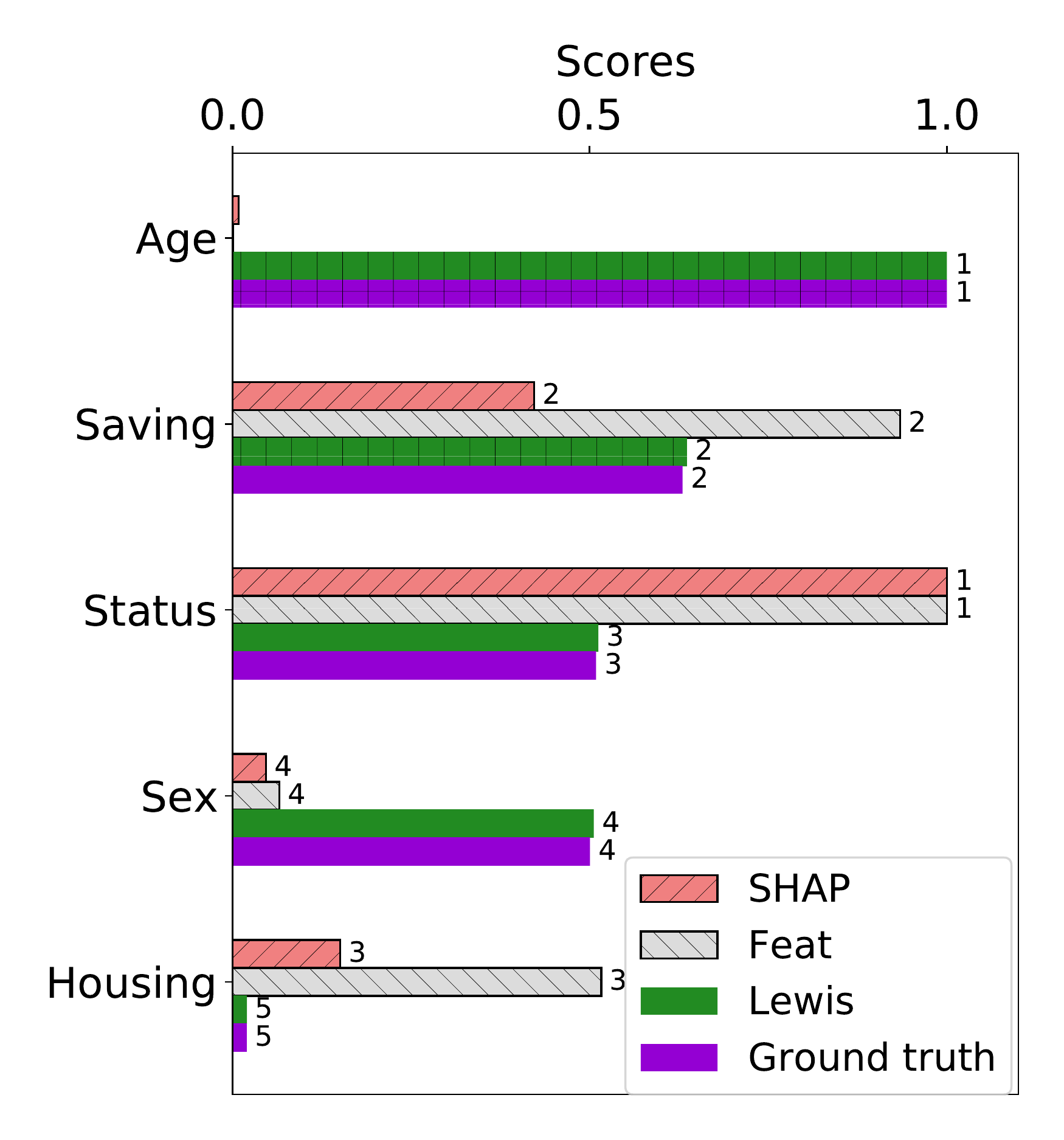}}
         \subcaptionbox{{Effect of sample size on error.} \label{fig:exp:synthetic:sample}}
        {\includegraphics[width=.36\columnwidth]{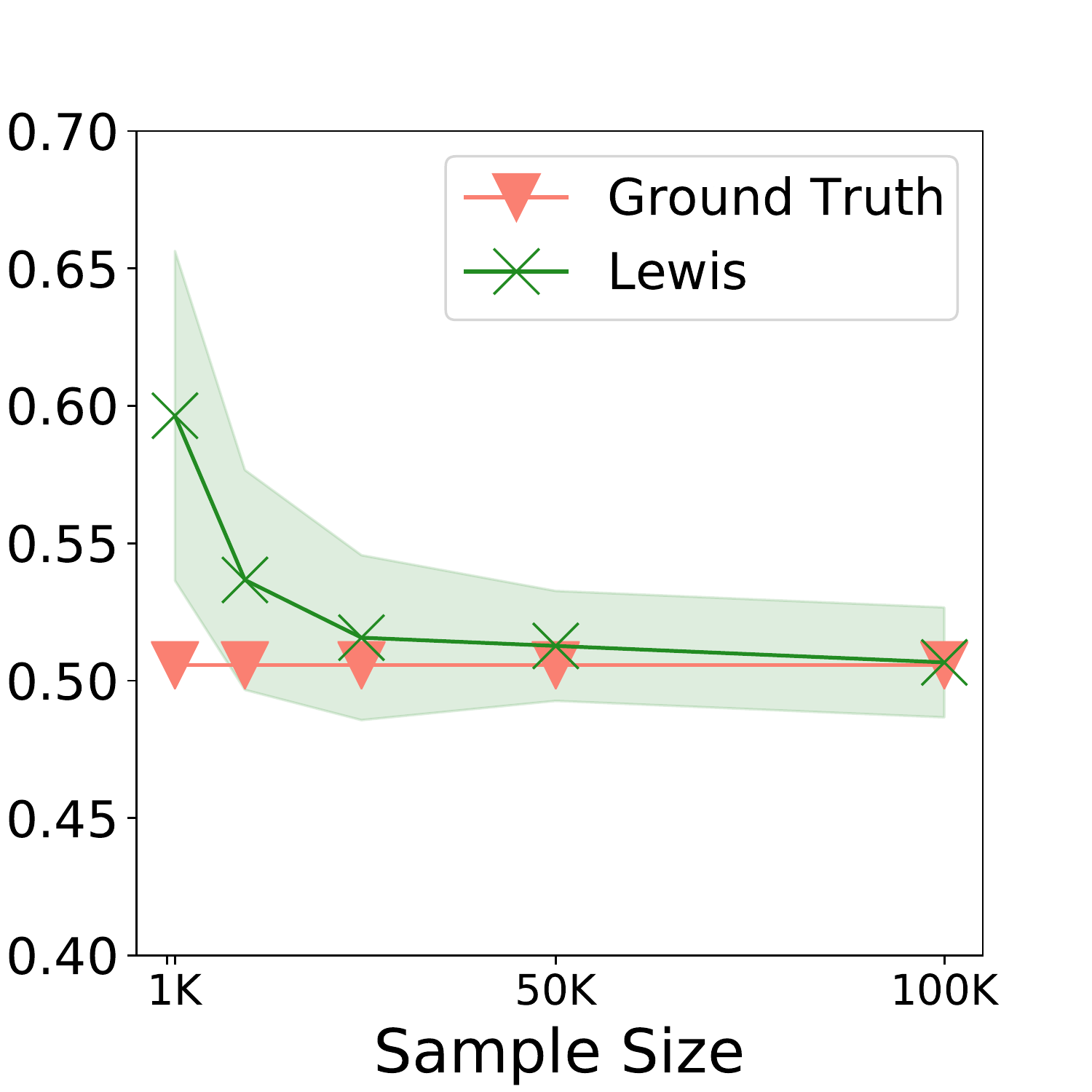}}
      \caption{Comparing with ground truth.}
      \label{fig:exp:related:synthetic}
\end{figure}
\noindent\textbf{Recourse analysis.} We sampled 1000 random instances that received negative outcomes and generated recourse (sufficiency threshold $\alpha = 0.9$) using \sys{}. Each unit change in attribute value was assigned unit cost. The output was evaluated with respect to the ground truth sufficiency and cost of returned actions. In all instances, \sys{}'s output achieved more than $0.9$ sufficiency with the optimal cost. This experiment validates the optimality of the IP formulation in generating effective recourse. To further test the \revb{scalability of \sys}, we considered a causal graph with $100$ variables and increased the number of actionable variables from $5$ to $100$. The number of constraints grew linearly from $6$ to $101$ (one for each actionable variable and one for the sufficiency constraint), and the running time increased from $1.65$ seconds to $8.35$ seconds, demonstrating \sys{}'s scalability to larger inputs.

\vspace{-0.3cm}
\section{Related Work and Discussion}
\label{sec:related}
Our research is mainly related to XAI work in quantifying feature importance and counterfactual explanations. 



\par
{\bf Quantifying feature importance.} Due to its strong axiomatic guarantees, methods based on Shapley values are emerging as the de facto approach for quantifying feature influence~\cite{lipovetsky2001analysis, vstrumbelj2014explaining,
lundberg2017unified,lundberg2018consistent,datta2016algorithmic,merrick2019explanation,frye2019asymmetric,aas2019explaining}. However, several practical and epistemological issues have been identified with these methods\ignore{(including attempts such as~\cite{merrick2019explanation,datta2016algorithmic} that simulate interventional distributions by perturbation or using marginal distributions)}. These issues arise primarily because existing proposals for quantifying the marginal influence of an attribute do not have any causal interpretation in general and, therefore, can lead to incorrect and misleading explanations~\cite{merrick2019explanation,kumar2020problems,frye2019asymmetric}. Another popular method for generating local explanations is LIME (Local Interpretable Model-agnostic Explanations~\cite{ribeiro2016should}, which trains an interpretable classifier (such as linear regression) on an instance obtained by perturbing that instance to be explained around its neighborhood. Several issues with LIME have also been identified in the literature, including its lack of human interpretability, its sensitivity to the choice of local perturbation, and its vulnerability to adversarial attacks~\cite{alvarez2018robustness, lundberg2017unified,molnar2020interpretable,slack2020fooling}. 

Unlike existing methods, our proposal offers the following advantages. (1) It is grounded in causality and counterfactual reasoning, captures insights from the theoretical foundation of explanations in philosophy, epistemology and social science, and can provide provably correct explanations. It has been argued that humans are selective about explanations and, depending on the context, certain contrasts are more meaningful than others~\cite{miller2019explanation,ducasse1926nature}. The notions of necessity and sufficiency have been shown to be strong criteria for preferred explanatory causes~\cite{greenland1999relation, robins1989probability, greenland1999epidemiology, tian2000probabilities, robertson1996common, cox1984probability, pearl2009causality}. (2) It accounts for indirect influence of attributes on algorithm's decisions; the problem of quantifying indirect influence has received scant attention in XAI literature (see \cite{adler2018auditing} for a non-causality-based approach). (3) It builds upon scores that are customizable and can therefore generate explanations at the global, contextual and local levels. (4) It can audit black-box algorithms merely by using historical data on its input and outputs. 



%

\par 
{\bf Counterfactual explanations.}
Our work is also related to a line of research that leverages counterfactuals to explain ML algorithm predictions~\cite{wachter2017counterfactual,laugel2017inverse,karimi2019model,ustun2019actionable,mahajan2019preserving,mothilal2020explaining}. In this context, the biggest challenge is generating explanations that
follow natural laws and are feasible and actionable in the real world.  Recent work attempts to address feasibility use ad hoc constraints \cite{ustun2019actionable,mothilal2020explaining,fariha2020,joshi2019towards,dhurandhar2018explanations,van2019interpretable,liu2019generative}. However, it has been argued that feasibility is fundamentally a causal concept \cite{barocas2020hidden,mahajan2019preserving,karimi2019model}.  Few attempts have been made to 
develop a causality-based approach that can generate actionable recourse by relying on the strong assumption that the underlying probabilistic causal model is fully specified or can be learned from data~\cite{mahajan2019preserving,karimi2019model,karimi2020algorithmic}. Our framework  extends this line of work by (1) formally defining feasibility in terms of probabilistic contrastive counterfactuals, and (2) providing a theoretical justification for taking a fully non-parametric approach for computing contrastive counterfactuals from historical data, thereby making no assumptions about the internals of the decision-making algorithm and the structural equations in the underlying probabilistic causal models. The IP formulation we used to generate actionable recourse is similar to \cite{ustun2019actionable} with a difference that their approach uses classifier parameters to bound the change in prediction as opposed to our sufficiency score based constraint. Our formulation is not only causal but also independent of the internals of the black-box algorithm.

\par{\bf Logic-based methods.} Our work shares some similarities with recent work in XAI that employs tools from logic-based diagnosis and operates with the logical representations of
ML algorithms~\cite{shih2018symbolic,ignatiev2020towards,darwiche2020reasons}. In this context, the fundamental concepts of prime implicate/implicant are closely related to sufficiency and necessary causation when the underlying causal model is a {\em logical circuit} \cite{hopkins2003clarifying,de1992characterizing,halpern2005causes,halpern2005causes,darwiche1994symbolic}. It can be shown that the notion of sufficient/necessary explanations proposed in ~\cite{shih2018symbolic} translates to explanations in terms of a set of attributes that have a sufficiency/necessary score of 1. However, these methods can generate explanations only in terms of a set of attributes, are intractable in model-agnostic settings, fail to account for the causal interaction between  attributes, and cannot go beyond deterministic algorithms. 

\par{\bf Algorithmic fairness.} The critical role of causality and background knowledge is recognized and acknowledged in the algorithmic fairness literature~\cite{kusner2017counterfactual,kilbertus2017avoiding,nabi2018fair,russell2017worlds,galhotra2017fairness,salimi2020causal,salimi2019interventional,salimi2020database}. In this context, contrastive counterfactuals have been used to capture individual-level fairness~\cite{counterfactualfairness}. It is easy to show that the notion of {\em counterfactual fairness} in ~\cite{counterfactualfairness} can be captured by the explanation scores introduced in this paper provided that an algorithm is counterfactually fair w.r.t. a protected attribute if the sufficiency score and necessity score of the sensitive attribute are {\em both} zero. Hence, \sys\ is useful for reasoning about individual-level fairness and discrimination.

\revc{Orthogonal to our work, \textit{strategic classification} addresses devising techniques that are robust to manipulation and gaming; recent literature has focused on studying the causal implications of such behavior~\cite{pmlr-v119-miller20b}. In the future, we plan to incorporate such techniques that make our system robust with respect to gaming.}  The metrics we introduce here for quantifying the necessity, sufficiency and necessity and sufficiency of an algorithm's input for its decision are adopted from the literature on probability of causation~~\cite{greenland1999relation, robins1989probability, greenland1999epidemiology, tian2000probabilities, robertson1996common, cox1984probability, pearl2009causality}. The results developed in Section~\ref{sec:identif} generalize and subsume earlier results from~\cite{tian2000probabilities,pearl2009causality} and substantially simplify their proofs.

\par
{\bf Assumptions and limitations.}  Our framework relies on two main assumptions to estimate and bound explanation scores, namely, the  availability of (1) data that is a representative sample of the underlying population of interest, and (2)  knowledge of the underlying causal diagram.  Dealing with non-representative samples goes beyond the scope of this paper, but there are standard approaches that can be adopted (see, e.g.,~\cite{bareinboim2012controlling}).  Furthermore, \sys\ is designed to work with any level of user's background knowledge. If no background knowledge is provided, \sys\ assumes no-confounding, i.e.,  $\pr(o \mid \Do(\mb \mb x), \mb k)=\pr(o \mid \mb \mb x, \mb k)$ and monotonicity.  Under these assumptions, the necessity score and sufficiency score, respectively, become $\frac{  \pr(o' \mid \mb x',\mb k)  - \pr(o' \mid \mb x,\mb k) }{\pr(o \mid \mb x,   \mb k)}$ and $\frac{\pr(o\mid  \mb x ,\mb k) - \pr(o\mid \mb x', \mb k) }{\pr(o' \mid  \mb x', \mb k)}$. The former can be seen as a group-level {\em attributable fraction}, which is widely used in epidemiology as a measure of the proportion of cases attributed to a particular risk factor~\cite{poole2015history}; the latter can be seen as a group-level {\em relative risk}, which is widely used in epidemiology to measure the risk of contracting a disease in a group exposed to a risk factor~\cite{khoury1989measurement}. When computed for individuals, these quantities can be interpreted as proportional to the difference between the ratio of positive/negative algorithmic decisions for individuals that are {\em similar} on all attributes except for $X$. In other words, the quantities measure the correlation between $X$ and the algorithm's decisions across similar individuals. \revc{ This correlation can be interpreted causally only under the no-confounding and monotonicity assumptions. {\em Nonetheless, quantifying the local influence of an attribute by measuring its correlation with an  algorithm's decision across similar individuals underpins most existing methods for generating local explanations such as Shapley values based methods~\cite{lundberg2017unified,aas2019explaining}, feature importance~\cite{vstrumbelj2014explaining}
, and LIME~\cite{ribeiro2016should}}. Approaches differ in terms of how they measure this correlation.} 

In principle, background knowledge on underlying causal models is required to generate effective and actionable explanations. {\em While this may be considered a limitation of our approach, we argue that all existing XAI methods either explicitly or implicitly make causal assumptions (such as those mentioned above in addition to feature independence and the possibility of simulating interventional distributions by perturbing data or using marginal distributions).} Hence, our framework replaces assumptions that are unrealistic with assumptions about the underlying causal diagram that need not be perfect to obtain valuable insights, can be validated using historical data and background knowledge~\cite{pearl2009causality}, and can be learned from a mixture of historical and interventional data~\cite{glymour2019review}. As argued above, in the worst case, our assumptions about generating local explanations are similar to those of existing work. Nevertheless, we show empirically in Section~\ref{sec:exp} that our methods are robust to slight violations of underlying assumptions and generate insights considerably beyond state-of-the-art methods in XAI.

\newpage
\bibliographystyle{plain}
\bibliography{ref}

\end{document}